\documentclass[a4paper,12pt]{amsart}
\usepackage{psfrag,graphicx}
\usepackage{psfrag}
\usepackage{anysize}
\newcommand{\new}{\newcommand}

\theoremstyle{plain}
\newtheorem{lemma}{Lemma}
\newtheorem{prop}{Proposition}
\newtheorem{corollary}{Corollary}
\newtheorem{theo}{Theorem}
\theoremstyle{definition}

\newtheorem{remark}{Remark}

\newtheorem{assumption}{Assumption}

\newtheorem{defi}{Definition}

\new{\argmin}{\operatornamewithlimits{argmin}}

\new{\runo}{{\mathbb R}}
\new{\cuno}{{\mathbb C}}

\new{\nat}{{\mathbb N}}

\new{\noi}{W}
\new{\ka}{\kappa}
\new{\la}{\lambda}
\new{\eps}{\varepsilon}

\new{\rp}{\la}

\new{\nor}[2]{\left\|{#2}\right\|_{#1}}
\new{\nory}[1]{{|#1|}}
\new{\norop}[3]{\left\|{#3}\right\|_{{}_{#1,#2}}}
\new{\scal}[3]{\left\langle{#2},{#3}\right\rangle_{#1}}
\new{\scaly}[2]{\left\langle{#1},{#2}\right\rangle}
\new{\sgn}[1]{\mathrm{sgn}(#1)}
\new{\supp}[1]{\mathrm{supp}(#1)}
\new{\Tr}[1]{\mathrm{Tr\,}(#1)}
\new{\Ker}[1]{\mathrm{ker\,}#1}

\new{\f}{f^*}
\new{\X}{\mathcal X}
\new{\Y}{\mathcal Y}
\new{\Z}{\mathcal Z}
\new{\G}{\Gamma}
\new{\g}{\gamma}
\new{\gG}{{\g\in\G}}
\new{\set}[1]{\{#1\}}
\new{\setG}[1]{({#1}_{\g})_{\gG}}
\new{\famG}[1]{(#1)_{\gG}}
\new{\be}{\beta}
\new{\bg}{\beta_\gamma}
\new{\vp}{\varphi}
\new{\vg}{\varphi_\gamma}
\new{\wg}{w_\gamma}
\new{\ug}{u_\gamma}
\new{\pg}{\psi_\gamma}

\new{\est}{\beta_n^\rp}

\new{\pen}[1]{p_\eps(#1)}
\new{\eqn}[1]{~(\ref{#1})}
\new{\beeq}[2]{\begin{equation}\label{#1}{#2}\end{equation}}

\new{\Pb}{{\mathbb P}}
\new{\PP}{P}
\new{\Ldue}{L^2_\Y(\PP)}
\new{\Lduen}{L^2_\Y({\mathbb P}_n)}
\new{\Lduex}{L^2_\Y(\X)}
\new{\ldue}{\ell_2}
\new{\hs}{{\mathcal L}_\text{HS}}
\new{\ma}{\rho}

\new{\Px}{\Phi_X}
\new{\Phx}{\Phi_x}
\new{\Ph}{\Phi}
\new{\PX}{\Phi_{\PP}}
\new{\Pn}{\Phi_n}

\new{\rph}{\hat\rp}

\new{\hh}{\mathcal H}

\new{\mir}{\beta^{\rp}}
\new{\II}{{\mathcal E}}
\new{\I}{{\mathcal E}^{\rp}}
\new{\In}{{\mathcal E}_n^\rp}
\new{\Ss}[2]{{\mathcal S}_{#1}\left(#2\right)}
\new{\Sh}[2]{{\mathbf S}_{#1}\left(#2\right)}

\new{\T}{{\mathcal T}}
\new{\beg}{\beta^\eps}
\new{\Ba}{H}
\new{\Bb}{M}
\new{\Prob}[1]{\operatornamewithlimits{\mathbb P}\left[#1\right]}
\new{\ex}[1]{\operatornamewithlimits{\mathbb E}\left[#1\right]}
\new{\exn}[1]{\operatornamewithlimits{{\mathbb E}_{n}}\left[#1\right]}
\new{\ls}{\kappa_0}
\new{\app}[1]{{\mathcal A}(#1)}
\new{\rpopt}{\rp_n^{\text{opt}} }

\new{\cs}{\cos\theta}
\new{\sn}{\sin\theta}
\new{\tn}{\tan\theta_n}
\new{\de}{\delta}
\new{\lao}{\cos{\de}}
\new{\muo}{\sin{\de}}
\new{\nullo}{{\mathcal P}_{\text{null}}}
\new{\sparso}{{\mathcal P}_{\text{sparse}}}
\new{\pieno}{{\mathcal P}_{\text{full}}}
\new{\dsparso}{{\Delta}_{\text{sparse}}}
\new{\dpieno}{{\Delta}_{\text{full}}}
\new{\bege}{\beta^\eps}
\new{\este}{\beta^{\eps,\rp}}
\new{\CC}{c}
\new{\note}[2]{{\color{red} #1}}
\new{\red}[1]{{\color{red}#1}}
\new{\blue}[1]{{\color{blue}#1}}
\new{\green}[1]{{\color{green}#1}}

\title{Elastic-Net Regularization in Learning Theory}
\author{C. De Mol, E. De Vito, L. Rosasco}
\address{Christine De Mol, Department of Mathematics and ECARES,
Universit\'e Libre de Bruxelles, Campus
   Plaine CP 217, Bd du Triomphe, 1050 Brussels, Belgium}
\email{demol@ulb.ac.be}
\address{Ernesto De Vito, D.S.A.,  Universit\`a di Genova, Stradone
  Sant'Agostino, 37, 16123, Genova, Italy  and INFN, Sezione di
  Genova, Via Dodecaneso 33, 16146 Genova, Italy}
\email{devito@dima.unige.it}
\address{Lorenzo Rosasco,
Center for Biological and Computational Learning at the 
Massachusetts Institute of Technology\\ \&
DISI, Universit\`a di Genova, Italy}
\email{lrosasco@mit.edu}
\date{\today}

\begin{document}
\begin{abstract}
{Within the framework of statistical learning theory we analyze in
detail the so-called elastic-net regularization scheme proposed by
Zou and Hastie \cite{zhuhas05} for the selection of groups of
correlated variables. To investigate on the statistical properties
of this scheme and in particular on its consistency properties, we
set up a suitable mathematical framework. Our setting is
random-design regression where we allow the response variable to
be vector-valued and we consider prediction functions which are
linear combination of elements ({\em features}) in an
infinite-dimensional dictionary. Under the assumption that the
regression function admits a sparse representation on the
dictionary, we prove that there exists a particular ``{\em
elastic-net representation}'' of the regression function such
that, if the number of data increases, the elastic-net estimator
is consistent not only for prediction but also for
variable/feature selection. Our results include finite-sample
bounds and an adaptive scheme to select the regularization
parameter. Moreover, using convex analysis tools, we derive an
iterative thresholding algorithm for computing the elastic-net
solution which is different from the optimization procedure
originally proposed in \cite{zhuhas05}.}
\end{abstract}

\maketitle

%%%%%%%%%%%%%%%%%%%%%%%%%%%%%%%%%%%%%%%%%%%%%%%%%%%%%%%%%%%
\section{Introduction}
%%%%%%%%%%%%%%%%%%%%%%%%%%%%%%%%%%%%%%%%%%%%%%%%%%%%%%%%%%%
We consider the standard framework of supervised learning, that is
nonparametric regression with random design. {In this
setting, there is an input-output pair $(X,Y)\in \X\times\Y$ with
unknown probability distribution $P$, and the goal is to find a
prediction function $f_n:\X\to\Y$, based on a training set
$(X_1,Y_1),\ldots,$ $(X_n,Y_n)$ of $n$ independent random pairs
distributed as $(X,Y)$. A good solution $f_n$ is such that, given
a new input $x\in\X$, the value $f_n(x)$ is a good prediction of
the true output $y\in\Y$. When choosing the square loss to measure
the quality of the prediction, as we do throughout this paper,
this means that the expected risk $\ex{\nory{Y-f_n(X)}^2}$ is {\em
small}, or, in other words, that $f_n$ is a {\em good}
approximation of the regression function $\f(x)=\ex{Y\mid X=x}$
minimizing this risk.}

In many learning problems, {a major goal besides prediction is
that of  {\em selecting the variables} that are {\em relevant to
achieve good predictions}.} In the problem of variable selection
we {are given} a set $\setG{\psi}$ of functions from the input
space $\X$ into the output space $\Y$ and we aim at selecting
those functions which are needed to represent the regression
function, where the {\em representation} is typically given by a
linear combination. The set $\setG{\psi}$ is usually called {\em
dictionary} and its elements {\em features}. We can think of the
features as measurements used to represent the input data, as
providing some relevant parameterization of the input space, or as
a (possibly overcomplete) dictionary of functions used to
represent the prediction function.  In modern applications, the
number $p$ of features in the dictionary is usually very large,
{possibly} much larger that the number $n$ of examples in the
training set. {This situation is often referred to as the ``large
$p$, small $n$ paradigm'' \cite{cantao07}, and a key to obtain a
meaningful solution in such case} is the requirement that the
prediction function $f_n$ is a linear combination of only a {\em
few} elements in the dictionary, i.e. that $f_n$ admits a {\em
sparse} representation.

The above setting can be illustrated by two examples of
applications we are currently working on and which provide an
underlying motivation for the theoretical framework developed in
the present paper. The first application is a classification
problem in computer vision, namely face detection
\cite{demodeveod07,demodeveod08,dedeodve08}. The training set contains images
of faces and non-faces and each image is represented by a very
large redundant set of features capturing the local geometry of
faces, for example wavelet-like dictionaries or other local
descriptors. The aim is to find a good predictor able to detect
faces in new images.

The second application is the analysis of microarray data, where
the features are the expression level measurements of the genes in
a given sample or patient, and the output is either a
classification label discriminating between two or more
pathologies or a  continuous index indicating,  for example,  the
gravity of an illness. In this problem, besides prediction of the
output for examples-to-come, another important goal is the
identification of the features that are the most relevant to build
the estimator and would constitute a gene signature for a certain
disease \cite{demotrve07,bamorove08}. In both applications, the
number of features we have to deal with is much larger than the
number of examples and assuming sparsity of the solution is a very
natural requirement.

{The problem of variable/feature selection has a long history in
statistics and it is known that the brute-force approach (trying
all possible subsets of features), though theoretically appealing,
is computationally unfeasible. {A first strategy to overcome this
problem is provided by greedy algorithms. A second route, which we
follow in this paper, makes use of sparsity-based regularization
schemes (convex relaxation methods). The most well-known example
of such schemes is probably the so-called  {\em Lasso regression}
\cite{tibshirani96} -- also referred to in the signal processing
literature as Basis Pursuit Denoising \cite{chdosa98} -- where a
coefficient vector $\be_n$ }is estimated as the minimizer
of the empirical risk penalized with the $\ell_1$-norm, namely
\[ \be_n=\argmin_{\be= (\bg)_\gG}\left(\frac{1}{n}\sum_{i=1}^n \nory{Y_i-f_\be(X_i)}^2+\rp
\sum_{\gG}|\bg|\right)\qquad\qquad ,\]
where $f_\be=\sum_{\gG}\bg\pg$, $\rp$ is a suitable positive
regularization parameter and $\famG{\pg}$ a given 
set of features. An extension of this approach, called {\em bridge
regression}, amounts to replacing the $\ell_1$-penalty by an
$\ell_p$-penalty \cite{fu98}. It has been shown that this kind of
penalty can still achieve sparsity when $p$ is bigger, but very
close to $1$ (see  \cite{kolt06}). For this class of techniques,
both consistency and computational aspects have been studied.
Non-asymptotic bounds within the  framework of statistical
learning have been studied in several  papers
\cite{kefu00,butswe06,loge02,tage06,green06,geer06,zhyu06,kolt06}.
 A common feature of these results is that they
assume that the dictionary is finite (with cardinality possibly
depending on the number of examples) and satisfies some
assumptions about the linear independence of the relevant features
--~see  \cite{kolt06} for a discussion on this point~-- whereas
$\Y$ is usually assumed to be $\runo$. {Several numerical
algorithms have also been proposed to solve the optimization
problem underlying Lasso regression and are based e.g. on
quadratic programming \cite{chdosa98}, on the so-called LARS
algorithm \cite{efhajo04} or on iterative soft-thresholding (see
\cite{daub} and references therein).}

}

Despite of its success in many applications, the Lasso strategy
has some drawback in variable selection problems where there are
highly correlated features and we need to identify all the
relevant ones. This situation is of uttermost importance for e.g.
microarray data analysis since, as well-known, there is a lot of
functional dependency between genes which are organized in small
interacting networks. The identification of such groups of
correlated genes involved in a specific pathology is desirable to
make progress in the understanding of the underlying biological
mechanisms.

Motivated by microarray data analysis, Zou and Hastie
\cite{zhuhas05} proposed the use of a penalty which is a weighted
sum of the $\ell_1$-norm and the square of the $\ell_2$-norm of
the coefficient vector $\be$. The first term enforces the sparsity of
the solution, whereas the second term ensures democracy among
groups of correlated variables. In \cite{zhuhas05} the
corresponding {method} is called {\em (naive) elastic net}. The
method allows to select groups of correlated features when the
groups are not known in advance (algorithms to enforce
group-sparsity with {\em preassigned} groups of variables have
been proposed in e.g. \cite{owen06,yulin06,forrau07} using other
types of penalties).

{In the present paper we study several properties of the
elastic-net regularization scheme for vector-valued regression in
a random design. In particular, we prove consistency under some
adaptive and non-adaptive choices for the regularization
parameter. As concerns variable selection, we assess the accuracy
of our estimator for the vector $\be$ with respect to the
$\ell_2$-norm, whereas the prediction ability of the corresponding
function {$f_n=f_{\be_n}$} is measured by the expected risk
$\ex{\nory{Y-f_n(X)}^2}$. To derive such error bounds, we
characterize the solution of the variational problem underlying
elastic-net regularization as the fixed point of a contractive map
and, as a byproduct, we derive an explicit iterative thresholding
procedure to compute the estimator. As explained below, in the
presence of highly collinear features, the presence of the
$\ell_2$-penalty, besides enforcing grouped selection, is crucial
to ensure stability with respect to random sampling.

In the remainder of this section, we define the main ingredients
for elastic-net regularization within our general framework,
discuss the underlying motivations for the method and then outline
the main results established in the paper.}

{As an extension of the setting originally proposed in
\cite{zhuhas05}, we allow the dictionary to have an infinite
number of features. In such case, to cope with infinite sums, we
need some assumptions on the coefficients. We assume that the
prediction function we have to determine is a linear combination
of the features $(\pg)_\gG$ in the dictionary and that the series
\[f_{\be}(x)=\sum_{\gG} \bg \pg(x), \]
converges absolutely for all $x\in\X$ and for all sequences
$\be=(\be_\g)_\gG$ satisfying $\sum_\gG\ug \bg^2 < \infty$, where
$\ug$ are given positive weights. The latter constraint can be
viewed as a constraint on the {\em regularity} of the functions
$f_\be$ we use to approximate the regression function. For
infinite-dimensional sets, as for example wavelet bases or
splines, suitable choices of the weights correspond to the
assumption that $f_\be$ is in a Sobolev space {(see Section 2 for
  more details about this point)}. Such requirement of regularity is
common when dealing with infinite-dimensional spaces of functions, as it happens in
approximation theory, signal analysis and inverse problems.

To ensure the convergence of the series defining $f_\be$,
we assume that
\begin{equation}\label{rrr}
\sum_{\gG} \frac{\nory{\pg(x)}^2}{\ug}\qquad\text{is finite for all
  }x\in X.
\end{equation}
Notice that for finite dictionaries, the series becomes a finite
sum and the previous condition as well as the introduction of
weights become superfluous.

To simplify the notation and the formulation of our results, and
without any loss in generality, we will in the following rescale
the features by defining $\vg=\pg / \sqrt{\ug}$, so that on this
{\em rescaled dictionary}, $f_\be=\sum_{\gG}\tilde{\bg}\vg$ will
be represented by means of a vector $\tilde{\bg}=\sqrt{\ug} \bg$
belonging to $\ldue$; the condition\eqn{rrr} then becomes
$\sum_{\gG} \nory{\vg(x)}^2<+\infty$, for all $x\in X$. From now
on, we will only use this rescaled representation and we drop the
tilde on the vector $\be$.

Let us now define our estimator as the minimizer of the empirical
risk penalized with a (weighted) elastic-net penalty, that is, a
combination of the squared $\ldue$-norm and of a weighted
$\ell_1$-norm of the vector $\be$.} More precisely, we define the
elastic-net penalty as follows.

\begin{defi}
Given a family $\setG{w}$ of weights $\wg\geq 0$ and a parameter
$\eps\geq 0$, let $p_\eps:\ell_2\to [0,\infty]$ be defined as
\beeq{enetpen}{\pen{\be}=\sum_{\gG} (\wg |\bg|+\eps \bg^2 )}
which
can also be rewritten as $\pen{\be}=\nor{1,w}{\be}+\eps
\nor{2}{\be}^2$, where $\nor{1,w}{\be}=\sum_{\gG}\wg |\bg|$.
\end{defi}
The weights $\wg$ allow us to enforce more or less sparsity on
different groups of features. We assume that they are prescribed
in a given problem, so that we do not need to explicitly indicate
the dependence of $\pen{\be}$ on these weights. The elastic-net
estimator is defined by the following minimization problem.

\begin{defi}
Given $\rp >0$,  let $\In:\ldue\to [0,+\infty]$ be the empirical
risk penalized by the penalty $\pen{\be}$
\beeq{elnetfunc}{\In(\be)=\frac{1}{n}\sum_{i=1}^n \nory{Y_i-
f_{\be}(X_i)}^2 + \rp \pen{\be},} 
and let $\est\in\ldue$ be {\em
the} or {\em a} minimizer of \eqn{elnetfunc} on $\ldue$
\beeq{estimatore}{\est=\argmin_{\be\in\ldue}\In(\be). }
\end{defi}
\noindent {The positive parameter $\rp$ is a regularization parameter
controlling the trade-off between the empirical error and the
penalty. Clearly, $\est$ also depends on
the parameter $\eps$, but we do not write explicitly this
dependence since $\eps$ will always be fixed.}

Setting $\eps=0$ in \eqn{elnetfunc}, we obtain as a special case
an infinite-dimensional extension of the Lasso regression scheme.
On the other hand, setting $\wg=0,\ \forall \g $, the method
reduces to $\ell_2$-regularized least-squares regression -- also
referred to as {\em ridge regression} -- with a generalized linear
model. The $\ell_1$-penalty has selection capabilities since it
enforces sparsity of the solution, whereas the $\ell_2$-penalty
induces a linear shrinkage on the coefficients leading to stable
solutions. The positive parameter $\eps$ controls the trade-off
between the $\ell_1$-penalty and the $\ell_2$-penalty.

{We will show that, if $\eps>0$, the minimizer $\est$ always exists and
is unique. In the paper we will focus on the
case $\eps>0$. Some of our results, however, still hold for
$\eps=0$, possibly under some supplementary conditions, as will be
indicated in due time.}

As previously mentioned one of the main advantage of the
elastic-net penalty is that it allows to achieve stability with
respect to random sampling. To illustrate this property more
clearly, we consider a toy example where the (rescaled) dictionary
has only two elements $\vp_1$ and $\vp_2$ with  weights
$w_1=w_2=1$.  The effect of random sampling is particularly
dramatic in the presence of highly correlated features. {To
illustrate this situation, we assume that $\vp_1$ and $\vp_2$
exhibit a special kind of linear dependency, namely that they are
linearly dependent on the input data $X_1,\ldots,X_n$:
$\vp_2(X_i)=\tn\ \vp_1(X_i)$ for all $i=1,\ldots,n$, where we have
parametrized the coefficient of proportionality by means of the
angle $\theta_n\in[0,\pi/2]$. Notice that this angle is a random
variable since it depends on the input data.}

Observe that the minimizers of \eqn{elnetfunc} must lie at a
tangency point between a level set of the empirical error and a
level set of the elastic-net penalty. The level sets of the
empirical error are all parallel straight lines with slope
$-\cot{\theta_n}$, as depicted by a dashed line in the two panels
of Figure~\ref{fig:2D}, whereas  the level sets of the elastic-net
penalty  are {\em  elastic-net balls} ({\em  $\eps$-balls}) with
center at the origin and corners at the intersections with the
axes, as depicted in Figure~\ref{fig:balls}.
\begin{figure}
  \centering 
 \includegraphics[scale=0.5]{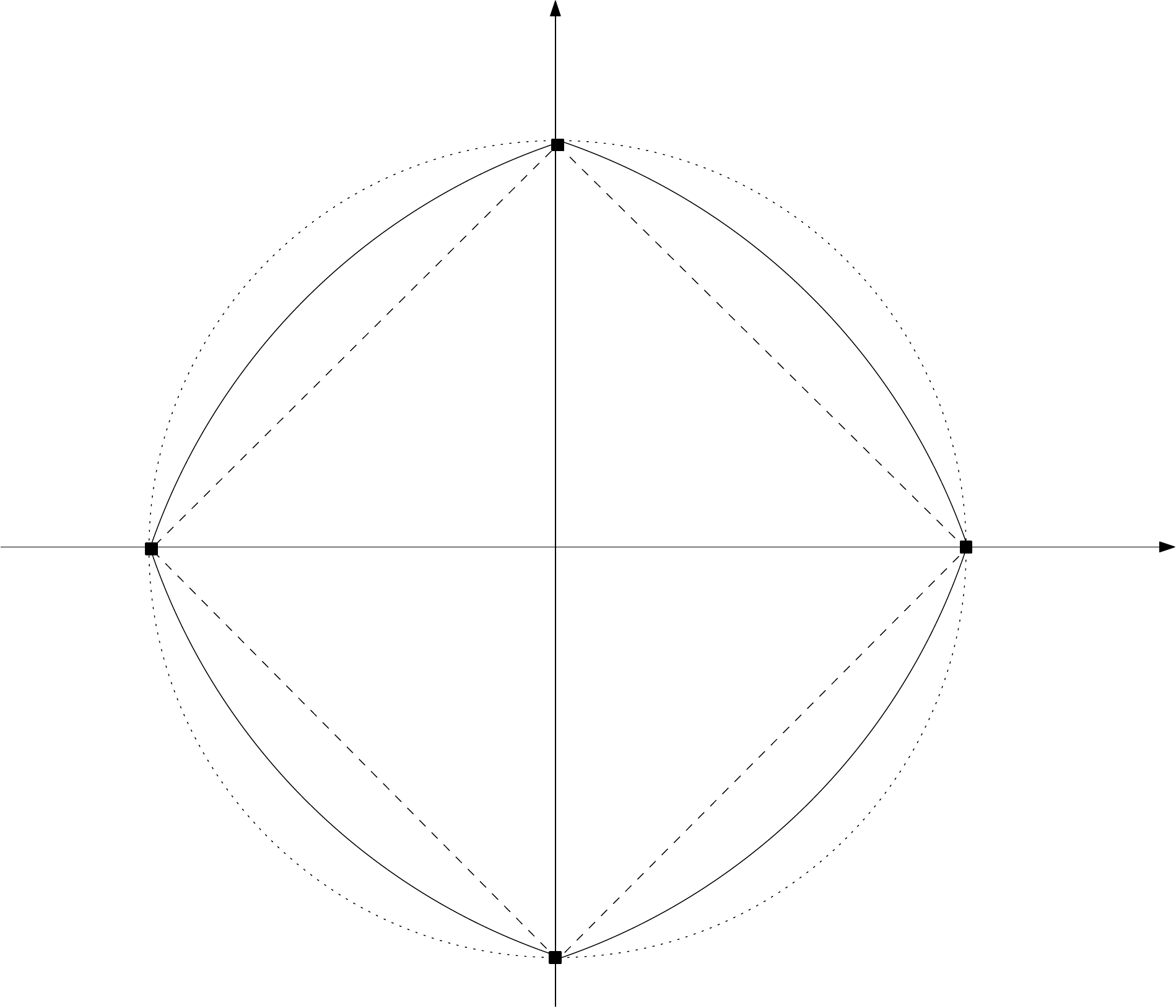}
  \caption{\small{The $\eps$-ball with $\eps>0$ (solid line),
      the square ($\ell_1$-ball), which is the $\eps$-ball with $\eps=0$
      (dashed line), and the disc ($\ell_2$-ball), which is the
      $\eps$-ball with $\eps\to\infty$ (dotted line).}}\label{fig:balls}
\end{figure}
When $\eps=0$, i.e. with a pure $\ell_1$-penalty (Lasso), the
$\eps$-ball is simply {\em a square} {(dashed line in Figure~\ref{fig:balls})} and we see that the unique
tangency point will be the {\em top corner} if $\theta_n>\pi/4$
(the point $T$ in the two panels of Figure~\ref{fig:2D}), or the
{\em right corner} if $\theta_n<\pi/4$. For $\theta_n = \pi/4$
~(that is, when $\vp_1$ and $\vp_2$ coincide on the data), the
minimizer of \eqn{elnetfunc} is no longer unique since the level
sets will touch along an edge of the square. Now, if $\theta_n$
randomly tilts around $\pi/4$  (because of the random sampling of
the input data), we see that the Lasso estimator is not stable
since it randomly jumps between the top and the right corner. If
$\eps\to\infty$, i.e. with a pure $\ell_2$-penalty (ridge
regression), the $\eps$-ball becomes {a disc (dotted line in
  Figure~\ref{fig:balls})}  and the minimizer is the point of the
straight line having minimal 
distance from the origin (the point $Q$ in the two panels of
Figure~\ref{fig:2D}). The solution always exists, is stable under
random perturbations, but it is never sparse (if
$0<\theta_n<\pi/2$).

{The situation changes if we consider the elastic-net estimator
with $\eps>0$ (the corresponding minimizer is the point $P$ in the
two panels of Figure~\ref{fig:2D}). The presence of the 
$\ell_2$-term ensures a smooth and stable behavior when the Lasso estimator
becomes unstable. More precisely, let $-\cot{\theta_+}$ be the
slope of the right tangent at the top corner of the elastic-net
ball ($\theta_+ > \pi/4$), and $-\cot{\theta_-}$ the slope of the
upper tangent at the right corner ($\theta_- < \pi/4$). As
depicted in top panel of Figure~\ref{fig:2D}, the minimizer will
be the top corner  if $\theta_n>\theta_+$. It will be the right
corner if $\theta_n<\theta_-$. In both cases the elastic-net
solution is sparse. On the other hand, if $\theta_-\leq\theta_n
\leq\theta_+$ the minimizer has both components $\be_1$ and
$\be_2$ different from zero -- see the bottom panel of
Figure~\ref{fig:2D}; in particular, $\be_1=\be_2$ if
$\theta_n=\pi/4$. Now we observe that if $\theta_n$ randomly tilts
around $\pi/4$, the solution smoothly moves between the top corner
and the right corner. However, the price we paid to get such
stability is a decrease in sparsity, since the solution is sparse
only when $\theta_n\not\in [\theta_-,\theta_+]$.}

\begin{figure}
\centering
\includegraphics[scale=.8]{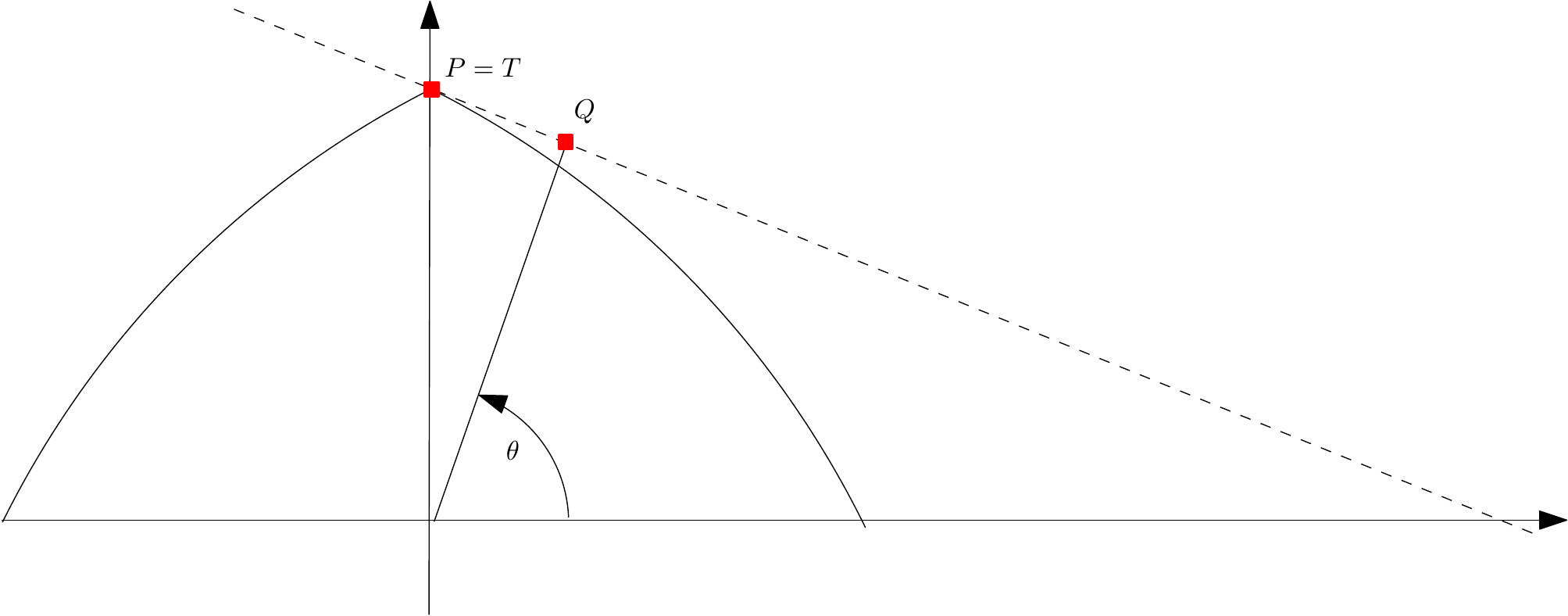}\\
\includegraphics[scale=.8]{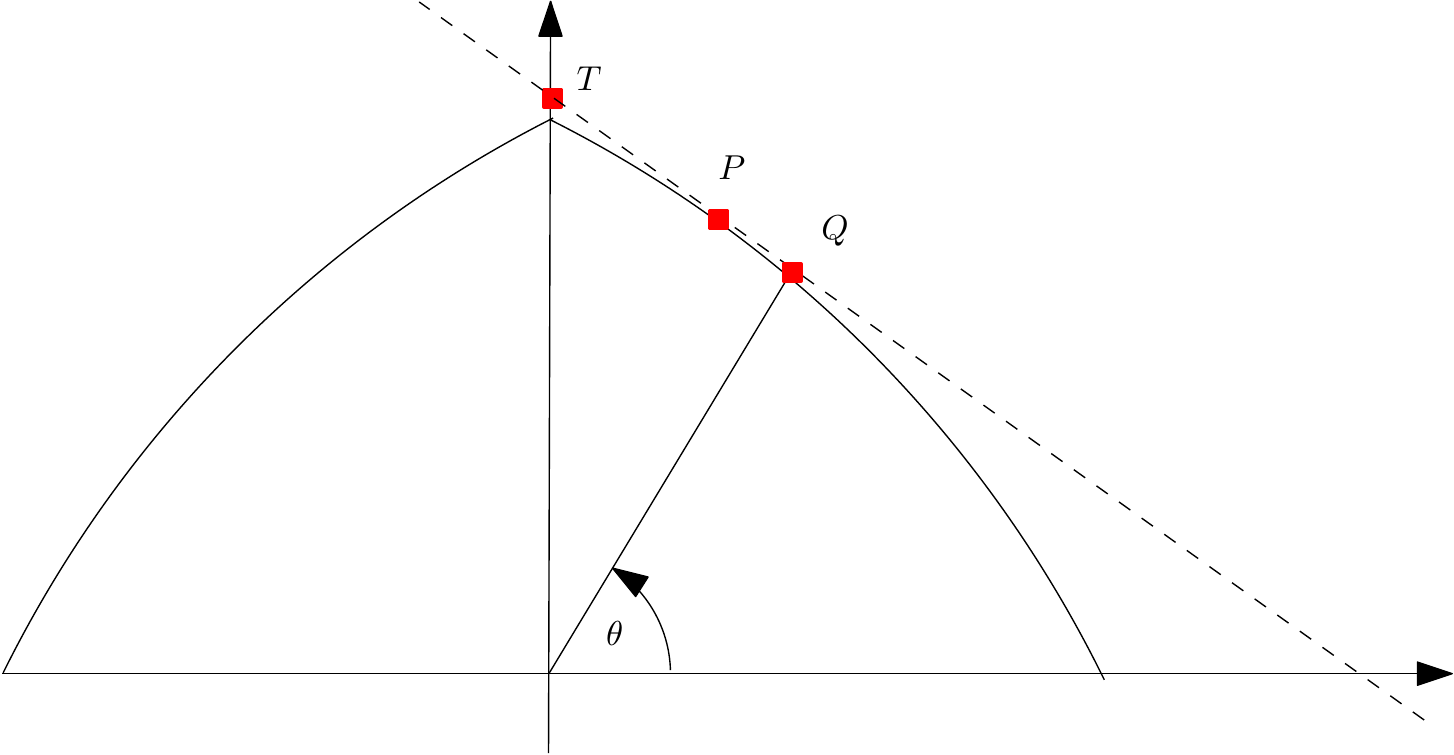}
\begin{caption}{\small{Estimators in the two-dimensional example: $T$=Lasso, $P$=elastic net
    and $Q$=ridge regression. Top panel:
    $\theta_+<\theta<\pi/2$. Bottom panel: $\pi/4<\theta<\theta_+$.}}\label{fig:2D}
\end{caption}
\end{figure}

{The previous elementary example could be refined in various ways
to show the essential role played by the $\ell_2$-penalty to
overcome the instability effects inherent to the use of the
$\ell_1$-penalty for variable selection in a random-design
setting.

We now conclude this introductory section by a summary of the main
results which will be derived in the core of the paper. A key
result will be to show that for $\eps>0$, $\est$ is the fixed
point of the following contractive map}

$$
\be=
 \frac{1}{\tau +\eps\rp}\Sh{\rp}{(\tau I -\Pn^* \Pn)\be+ \Pn^* Y}
$$
where $\tau$ is a suitable relaxation constant, $\Pn^*\Pn$ is the
matrix with entries
$(\Pn^*\Pn)_{\g,\g^\prime}=\frac{1}{n}\sum_{i=1}^n\scaly{\vp_\g(X_i)}{\vp_{\g^\prime}(X_i)}$,
$\Pn^*Y$ is the vector $(\Pn^*Y)_\g=
\frac{1}{n}\sum_{i=1}^n\scaly{\vp_\g(X_i)}{Y_i}$
($\scaly{\cdot}{\cdot}$ denotes the scalar product in the output
space $\Y$). Moreover, $\Sh{\rp}{\be}$ is the soft-thresholding
operator acting componentwise as follows
\[\ [\Sh{\rp}{\be}]_\g=\left\{\begin{array}{lccl} \bg -\frac{\rp\wg}{2} &
      \text{\rm if } & \bg>\frac{\rp \wg}{2}  \\
0 & \text{\rm if } & |\bg|\leq \frac{\rp\wg}{2} \\
\bg+\frac{\rp\wg}{2} & \text{\rm if } & \bg <-\frac{\rp\wg}{2}
\end{array}\right..\]
As a consequence of the Banach fixed point theorem, $\est$ can be
computed by means of an iterative algorithm. This procedure is
completely different from the modification of the LARS algorithm
used in \cite{zhuhas05} and is akin instead to the algorithm
developed in \cite{daub}.

Another interesting property which we will derive from the above
equation is that the non-zero components of $\est$ are such that
$\wg\leq\frac{C}{\rp}$, where $C$ is a constant depending on the
data. Hence the only active features are those for which the
corresponding weight lies below the threshold $C/\rp$. If the
features are organized into finite subsets of increasing
complexity (as it happens for example for wavelets) and the
weights tend to infinity with increasing feature complexity, then
the number of active features is finite and can be {determined}
for any given data set. Let us recall that in the case of ridge
regression, the so-called {\em representer theorem}, see
\cite{wahba90}, ensures that {we only have to solve in practice} a
finite-dimensional optimization problem, even when the dictionary
is infinite-dimensional (as in kernel methods). This is no longer
true, however, with an $\ell_1$-type regularization and, {for
practical purposes, one would need to truncate infinite
dictionaries. A standard way to do this is to consider only a
finite subset of $m$ features, with $m$ possibly depending on $n$
-- see for example \cite{butswe06,green06}. Notice that such
procedure implicitly assumes some order in the features and makes
sense only if the retained features are the most relevant ones.}
For example, in \cite{bacoda06}, it is assumed that there is a
natural exhaustion of the hypothesis space with nested subspaces
spanned by finite-dimensional subsets of features of increasing
size. In our approach we adopt a different strategy, namely the
encoding of such information in the elastic-net  penalty by means of
suitable weights in the $\ell_1$-norm.

The main result of our paper concerns the consistency for variable
selection of $\est$. We prove that, if the regularization
parameter $\rp=\la_n$ satisfies the conditions {$\lim_{n\to\infty}\rp_n=0$ and } $\lim_{n\to\infty}(\la_n\sqrt{n} -2\log n)= +\infty$, then
$$
\lim_{n\to \infty}  \nor{2}{\be_n^{\rp_n}-\beg} =0 \qquad\text{
with probability one},
$$
where the vector $\beg$, which we call the {\em elastic-net
representation of $f_\be$}, is the minimizer of
\[\min_{\be\in\ldue} \left(\sum_{\gG} \wg |\bg| + \eps
\sum_{\gG} |\bg|^2 \right)\qquad\text{subject to}\qquad
f_\be=f^*.\] The vector $\beg$ exists and is unique provided that
$\eps>0$ and the regression function $\f$ admits a {\em sparse
representation on the dictionary}, i.e.  $\f=\sum_{\gG}\bg^*\vg$
for at least a vector $\be^*\in\ldue$ such that
$\sum_{\gG}\wg|\bg^*|$ is finite. Notice that,  when the features
are linearly dependent, there is a problem of identifiability
since there are many vectors $\be$ such that $\f=f_\be$. The
elastic-net regularization scheme forces $\be_n^{\rp_n}$ to
converge to $\beg$. This is precisely what happens for linear
inverse problems where the regularized solution converges to the
minimum-norm solution of the least-squares problem. As a
consequence of the above convergence result, one easily deduces
the consistency of the corresponding prediction function
$f_n:=f_{\be^{\la_n}_n}$, that is, $\lim_{n\to \infty}
\ex{\nory{f_n-f^*}^2}=0$ with probability one. When the regression
function does not admit a sparse representation, we can still
prove the previous consistency result for $f_n$ provided that the
linear span of the features is sufficiently rich.
Finally, we use a data-driven choice for the regularization
parameter, based on the so-called balancing principle
\cite{bauper05}, to obtain non-asymptotic bounds which are
adaptive to the unknown regularity of the regression function.

 {The rest of the paper is organized as follows. In Section
2, we set up the mathematical framework of the problem. In Section
3, we analyze the optimization problem underlying elastic-net
regularization and the iterative thresholding procedure we propose
to compute the estimator. Finally, Section 4 contains the
statistical analysis with our main results concerning the
estimation of the errors on our estimators as well as their
consistency properties under appropriate a priori and adaptive
strategies for choosing the regularization parameter.}

%%%%%%%%%%%%%%%%%%%%%%%%%%%%%%%%%%%%%%%%%%%%%%%%%%%%%%%%%%%
\section{Mathematical setting of the problem}\label{preli}
%%%%%%%%%%%%%%%%%%%%%%%%%%%%%%%%%%%%%%%%%%%%%%%%%%%%%%%%%%%

%%%%%%%%%%%%%%%%%%%%%%%%%%%%%%%%%%%%%%%%%%%%%%%%%%%%%%%%%%%
\subsection{Notations and assumptions}
%%%%%%%%%%%%%%%%%%%%%%%%%%%%%%%%%%%%%%%%%%%%%%%%%%%%%%%%%%%

In this section we describe the general setting of the regression
problem we want to solve and specify all the required assumptions.

We assume that $\X$ is a separable metric space and that $\Y$ is a
(real) separable Hilbert space, with norm and scalar product
denoted respectively by $\nory{\cdot}$ and $\scaly{\cdot}{\cdot}$.
Typically, $\X$ is  a subset of $\mathbb R^d$ and $\Y$ is $\mathbb
R$.  Recently, however, there has been an increasing interest for
vector-valued regression problems \cite{mipo05,bagibama08} and
multiple supervised learning tasks \cite{mipo05a,arevpo07}: in
both settings $\Y$ is taken to be $\mathbb R^m$. {Also
infinite-dimensional output spaces are of interest as e.g. in the
problem of estimating of glycemic response during a time interval
depending on the amount and type of food; in such case, $\Y$ is
the space $L^2$ or some Sobolev space. Other examples of
applications in an infinite-dimensional setting are given in
\cite{camipoyi08}.}

Our first assumption concerns the set of features.
\begin{assumption}
The family of features $\famG{\vg}$ is a countable
set of measurable functions $\vg:\X\to\Y$ such that
\begin{eqnarray}
  \label{eq:rkhs}
\forall  x\in \X \qquad k(x)=\sum_{\gG} \nory{\vg(x)}^2\leq \kappa,
\end{eqnarray}
for some finite number $\kappa$.
\end{assumption}
The index set $\G$ is countable, but we do not assume any order.
{As for} the convergence of series, we use the notion of
summability: given a family $\setG{v}$ of vectors in a normed
vector space $V$, $v=\sum_\gG v_\g$ means that $\setG{v}$ is summable\footnote{That is, for
  all $\eta>0$, there is a finite subset $\G_0\subset\G $ such that
$\nor{V}{v-\sum_{\g\in\G'}v_\g} \leq\eta$ for all finite subsets
$\G'\supset \G_0$. If $\G=\nat$, the notion of summability is
equivalent to requiring the series to converge unconditionally
(i.e. its terms can be permuted without affecting convergence). If
the vector space is finite-dimensional, summability is equivalent
to absolute convergence, but in the infinite-dimensional setting,
there are summable series which are not absolutely convergent.} 
with sum $v\in V$. 

Assumption~1 can be seen as a condition on the class of functions
that can be {recovered} by the elastic-net scheme. {As already
noted in the Introduction, we have at our disposal an arbitrary
(countable) dictionary $\famG{\pg}$ of measurable
functions, and we try to approximate $\f$ with linear combinations
$f_\be(x)=\sum_{\gG}\bg\pg(x)$ where the set of coefficients
$(\bg)_{\gG}$ satisfies {some {\em decay condition} equivalent to
a {\em regularity condition} on the functions $f_\be$}. We make
this condition precise by assuming that there exists a sequence of
positive weights $\setG{u}$ such that $\sum_{\gG}\ug \bg^2<\infty$
and, for any of such vectors $\be=(\bg)_{\gG}$, that the series
defining $f_\be$ converges absolutely for all $x\in\X$. These two
facts follow from the requirement that the set of rescaled
features $\vg=\frac{\pg}{\sqrt{\ug}}$ satisfies
$\sum_{\gG}|\vg(x)|^2<\infty$. Condition\eqn{eq:rkhs} is a little
bit stronger since it requires that
$\sup_{x\in\X}\sum_{\gG}|\vg(x)|^2<\infty$, so that we also have
that the functions $f_\be$ are bounded. To simplify the notation,
in the rest of the paper, we only use the (rescaled) features
$\vg$ and, with this choice, the regularity condition on the
coefficients $(\bg)_{\gG}$ becomes $\sum_{\gG}\bg^2<\infty$. }

An example of features satisfying the condition\eqn{eq:rkhs} is
given by a family of {{\em rescaled}} wavelets on $\X=[0,1]$.
Let $\set{\psi_{jk}\mid j=0,1\ldots;\ k\in\Delta_j}$ be a
{orthonormal} wavelet basis in $L^2([0,1])$ with regularity
$C^r$, $r>\frac{1}{2}$, where for $j\geq 1$ $\set{\psi_{jk}\mid
k\in\Delta_j}$ is the orthonormal wavelet basis  {(with
suitable boundary conditions)} spanning the detail space at level
$j$. To simplify notation, it is assumed that the set
$\set{\psi_{0k}\mid k\in\Delta_0}$ contains both the wavelets and
the scaling functions at level $j=0$. Fix $s$ such that
$\frac{1}{2}<s< r$ and let $\varphi_{jk}=2^{-js}\psi_{jk}$. Then
\[ {\sum_{j=0}^\infty\sum_{k\in\Delta_j} |\varphi_{jk}(x)}|^2=
{\sum_{j=0}^\infty\sum_{k\in\Delta_j}}2^{-2js}|\psi_{jk}(x)|^2\leq
C \sum_{j=0}^\infty2^{-2js} 2^j= C \frac{1}{1- 2^{1-2s}}=\kappa,\]
where $C$ is a suitable constant depending on the number of
wavelets that are non-zero at a point $x\in [0,1]$ for a given
level $j$, and on the maximum values of the scaling function and
of the mother wavelet; see \cite{amanpe06} for a similar setting.

Condition \eqn{eq:rkhs} allows to define the hypothesis space in
which we search for the estimator. Let $\ell_2$ be the Hilbert
space of the families $\setG{\be}$ of real numbers such that
$\sum_{\gG}\bg^2<\infty$, with the usual scalar product
$\scal{2}{\cdot}{\cdot}$ and the corresponding norm
$\nor{2}{\cdot}$.  We will denote by $\setG{e}$ the canonical
basis of $\ldue$ and by $\supp{\be}=\set{\gG\mid \bg \neq 0}$ the
support of $\be$. The Cauchy-Schwarz inequality and the condition
\eqn{eq:rkhs} ensure that, for any $\be=\setG{\be}\in\ldue$, the
series
\[\sum_{\gG}\bg \vg(x)=f_\be(x)\]
is summable in $\Y$ uniformly on $\X$ with
\beeq{csi}{\sup_{x\in\X}\nory{f_\be(x)}\leq
\nor{2}{\be}\ka^{\frac{1}{2}}.} Later on, in
Proposition~\ref{rkhs}, we will show that the hypothesis space
$\hh=\set{f_\be\mid\be\in\ldue}$ is then a vector-valued
reproducing kernel Hilbert space on $\X$ with a bounded kernel
\cite{cadeto06}, and that $\setG{\vp}$ is a {normalized tight}
frame {for $\hh$.}
In the example of the
wavelet features one can easily check that $\hh$ is the Sobolev
space $H^s$ on $[0,1]$ {and $\nor{2}{\be}$ is equivalent to $\nor{H^s}{f_\be}$.}

The second assumption concerns the regression model.
\begin{assumption}\label{statistical_model}
The random couple $(X,Y)$ in $\X\times \Y$ {obeys} the
regression model {
\[Y=\f(X)+\noi\]
where
\begin{equation}\label{eq:sparse1}
\f=f_{\be^*}\qquad \text{for some }\be^*\in\ldue\quad  \text{
with}\qquad  \sum_{\gG}\wg|\bg^*| < +\infty
\end{equation}}
and
\begin{eqnarray}
\label{eq:zeromean}
\ex{\noi\mid X} & = & 0 \\
\ex{\exp\left(\frac{\nory{\noi}}{L}\right)-\frac{\nory{\noi}}{L}-1\Big| X}& \leq &
\label{eq:noise}\frac{\sigma^2}{2L^2}
\end{eqnarray}
with $\sigma,L>0$. {The family $\famG{\wg}$ are the positive
weights defining the elastic-net  penalty $\pen{\be}$ in\eqn{enetpen}}.
\end{assumption}

Observe that $f^*=f_{\be^*}$  is always a bounded function
by\eqn{csi}. Moreover the condition\eqn{eq:sparse1} is a further
regularity condition on the regression function and will not be
needed for some of the results derived in the paper.
Assumption\eqn{eq:noise} is satisfied by bounded, Gaussian or
sub-Gaussian noise. In particular, it implies
\begin{eqnarray}
  \label{eq:bennett}
\ex{\nory{\noi}^m|X}\leq
  \frac{1}{2}m!\; \sigma^2 L^{m-2},\qquad \forall m\geq 2,
\end{eqnarray}
see \cite{vaart}, so that $W$ has a finite second moment.  It
follows that $Y$ has a finite first moment and\eqn{eq:zeromean}
implies that $\f$ is the regression function $\ex{Y\mid
  X=x}$.

Condition\eqn{eq:sparse1} controls both the sparsity and the
regularity of the regression function. If $\inf_{\gG}\wg=w_0>0$,
it is sufficient to require that $\nor{1,w}{\be^*}$ is finite.
Indeed, the H\"{o}lder inequality gives that \beeq{daneo}{
\nor{2}{\be} \leq\frac{1}{w_0}\nor{1,w}{\be}.} If $w_0=0$, we also
need $\nor{2}{\be^*}$ to be finite. In the example of the
{(rescaled)} wavelet features a natural choice for the weights is
$w_{jk}=2^{ja}$ for some $a\in\runo$, so that $\nor{1,w}{\be}$ is
equivalent to the norm $\nor{B^{\tilde{s}}_{1,1}}{f_\be}$, with
$\tilde{s}= a+s+\frac{1}{2}$, in the Besov space
$B^{\tilde{s}}_{1,1}$  on $[0,1]$ (for more details, see e.g. the
appendix in \cite{daub}). In such a case, \eqn{eq:sparse1} is
equivalent to requiring that $\f\in H^s\cap B^{\tilde{s}}_{1,1}$.

Finally, our third assumption concerns the training sample.
\begin{assumption}\label{sample_model}
The sequence of {random pairs $(X_n,Y_n)_{n\geq 1}$ are
independent and identically distributed ({\em i.i.d.}) according
to the distribution of $(X,Y)$.}
\end{assumption}
In the following, we let $\PP$ be the probability distribution of
$(X,Y)$, and $\Ldue$ be the Hilbert space of (measurable)
functions $f:\X\times\Y\to\Y$ with the norm
\[\nor{\PP}{f}^2=\int_{\X\times\Y}|f(x,y)|^2\ dP(x,y).\]
With a slight abuse of notation, we regard the random pair $(X,Y)$
as a function on $\X\times\Y$, that is, $X(x,y)=x$ and $Y(x,y)=y$.
Moreover, we denote by $\Pb_n=\frac{1}{n}\sum_{i=1}^n
\delta_{X_i,Y_i}$ the empirical distribution and by $\Lduen$ the
corresponding (finite-dimensional) Hilbert space with norm
\[\nor{n}{f}^2= \frac{1}{n} \sum_{i=1}^{n} \nory{f(X_i,Y_i)}^2.\]

%%%%%%%%%%%%%%%%%%%%%%%%%%%%%%%%%%%%%%%%%%%%%%%%%%%%%%%%%%%
\subsection{Operators defined by the set of features}\label{sec:op}
%%%%%%%%%%%%%%%%%%%%%%%%%%%%%%%%%%%%%%%%%%%%%%%%%%%%%%%%%%%

The choice of a quadratic loss function and the Hilbert structure
of the hypothesis space suggest to use some tools from the theory
of linear operators. In particular, the function $f_\be$ depends
linearly on $\be$ and can be regarded as an element of both
$\Ldue$ and of $\Lduen$. Hence it defines two operators, whose
properties are summarized by the next two propositions, based on
the following lemma.
\begin{lemma}\label{rankone}
For any fixed $x\in\X$, the map $\Phx:\ldue\to\Y$ defined by
\[ \Phx\be=\sum_{\gG} \vg(x) \bg=f_\be(x) \]
is a  Hilbert-Schmidt operator, its adjoint $\Phx^*:\Y\to\ldue$ acts as
\begin{equation}
  \label{eq:2}
(\Phx^* y)_\g=\scaly{y}{\vg(x)}\qquad \gG \quad y\in\Y.
\end{equation}
In particular $\Phx^*\Phx$ is a trace-class operator with
\beeq{trace}{\Tr{\Phx^*\Phx }\leq\ka.} Moreover, $\Px^*Y$ is
a $\ldue$-valued random variable with
\begin{equation}
  \label{eq:9}
\nor{2}{\Px^*Y}\leq \ka^{\frac{1}{2}} \nory{Y},
\end{equation}
and $\Px^*\Px$ is a $\hs$-valued random variable with
\begin{equation}
  \label{eq:pa}
\nor{\mathrm{HS}}{\Px^*\Px}\leq \ka,
\end{equation}
where $\hs$ denotes the separable Hilbert space of the
Hilbert-Schmidt operators on $\ell_2$, and
$\nor{\mathrm{HS}}{\cdot}$ is the Hilbert-Schmidt norm.
\end{lemma}

\begin{proof}
Clearly $\Phx$ is a linear map from $\ldue$ to $\Y$. Since
$\Phx e_\g=\vg(x)$, we have
$$
\sum_{\gG}\nory{\Phx e_\g}^2=\sum_{\gG}\nory{\vg(x)}^2\leq\ka,
$$
so that $\Phx$ is a Hilbert-Schmidt operator and
$\Tr{\Phx^*\Phx}\leq\ka$ {by\eqn{eq:rkhs}}. Moreover, given $y\in\Y$ and
$\gG$
\[(\Phx^* y)_\g=\scal{2}{\Phx^* y}{e_\g}=\scaly{y}{\vg(x)}\]
which is \eqn{eq:2}. Finally, since $\X$ and $\Y$ are separable,
the map $(x,y)\to\scaly{y}{\vg(x)}$ is measurable, then
$(\Px^*Y)_\g$ is a real random variable and, since $\ldue$ is
separable, $\Px^*Y$ is $\ldue$-valued random variable with
\[\nor{2}{\Px^*Y}^2=\sum_{\gG} \scaly{Y}{\vg(X)}^2\leq \ka \nory{Y}^2.\] A similar proof holds for
$\Px^*\Px$, recalling that any trace-class operator is in $\hs$
and $\nor{\mathrm{HS}}{\Px^*\Px}\leq\Tr{\Px^*\Px}$.
\end{proof}
The following proposition defines the distribution-dependent
operator $\PX$ as a map from $\ell_2$ into $\Ldue$.
\begin{prop}\label{lemma1}
The map
$\PX:\ldue\to\Ldue$, defined by $\PX\be=f_\be$, is a  Hilbert-Schmidt
operator and
\begin{eqnarray}
\PX^* Y & = & \ex{\Px^*Y} \label{dual_pop} \\
\PX^* \PX & = &\ex{\Px^*\Px} \label{eq:PXS} \\
\Tr{\PX^*\PX }  &= & \ex{k(X)}\leq\ka. \label {trace1}
\end{eqnarray}
\end{prop}
\begin{proof}
Since $f_\be$ is a bounded (measurable) function, $f_\be\in\Ldue$ and
\[\sum_{\gG}\nor{\PP}{\PX
  e_\g}^2=\sum_{\gG}\ex{\nory{\vg(X)}^2}=\ex{k(X)}\leq\ka. \]
Hence  $\PX$ is a Hilbert-Schmidt operator with
$\Tr{\PX^*\PX}=\sum_{\gG}\nor{\PP}{\PX
  e_\g}^2$ so that\eqn{trace1} holds.
By\eqn{eq:noise} $\noi$ has a finite second moment and by\eqn{csi}
$\f=f_{\be^*}$ is a bounded function, hence $Y=\f(X)+W$ is in
$\Ldue$. Now for any $\be\in\ldue$ we have
\[\scal{2}{\PX^*Y}{\be}=
\scal{\PP}{Y}{\PX\be}=\ex{\scaly{Y}{\Px\be}}=\ex{\scal{2}{\Px^*Y}{\be}}.\]
On the other hand, by\eqn{eq:9}, $\Px^*Y$ has finite expectation,
so that \eqn{dual_pop} follows. Finally, given $\be,\be'\in\ldue$
\[\scal{2}{\PX^*\PX\be'}{\be}=
\scal{\PP}{\PX\be'}{\PX\be}=\ex{\scaly{\Px\be'}{\Px\be}}=\ex{\scal{2}{\Px^*\Px\be'}{\be}}\]
so that\eqn{eq:PXS} is clear, since  $\Px^*\Px$ has finite expectation
as a consequence of the fact that it is a bounded
$\hs$-valued random variable.
\end{proof}
Replacing $\PP$ by the empirical measure we get the sample version
of the operator.
\begin{prop}\label{lemma0}
The map $\Pn:\ldue\to \Lduen$  defined by $\Pn\be=f_\be$
is  Hilbert-Schmidt operator and
\begin{eqnarray}\label{axs}
 \Pn^* Y & = & \frac{1}{n}\sum_{i=1}^n \Phi_{X_i}^*
  Y_i   \\
 \Pn^*\Pn & = & \frac{1}{n}\sum_{i=1}^n \Phi^*_{X_i}
  \Phi_{X_i} \label{pippo}\\
\Tr{\Pn^*\Pn} & = & \frac{1}{n}\sum_{i=1}^n k(X_i)\leq\ka \label{tracetx}.
\end{eqnarray}
\end{prop}

The proof of Proposition~\ref{lemma0} is analogous to the proof of
Proposition~\ref{lemma1}, except that $\PP$ is to be replaced by
$\Pb_n$.

By\eqn{eq:2} with $y=\vp_{\g'}(x)$, we have that the matrix
elements of the operator $\Phx^*\Phx$ are
$(\Phx^*\Phx)_{\g\g'}=\scaly{\vp_{\g'}(x)}{\vg(x)}$ so that
$\Pn^*\Pn$ {is the empirical mean of the Gram matrix of the set
$\setG{\vp}$, whereas $\PX^*\PX$ is the corresponding mean with
respect to the distribution $P$.}
Notice that if the
features are linearly dependent in $\Lduen$, the 
{matrix} $\Pn^*\Pn$ has a non-trivial kernel and hence is not
invertible. More important, if $\G$ is countably infinite,
$\Pn^*\Pn$ is a compact operator, so that its inverse (if it
exists) is not bounded. On the contrary, if $\G$ is finite and
$\setG{\vp}$ are linearly independent in $\Lduen$, then $\Pn^*\Pn$
is invertible. A similar reasoning holds for the
{matrix} $\PX^*\PX$. {To control whether}
{these matrices}
have a bounded inverse or not, we introduce a lower
spectral bound $\ls\geq 0$, such that
\[ \ls \leq \inf_{\be\in\ldue\,\mid\,
  \nor{2}{\be}=1}\scal{2}{\PX^*\PX\be}{\be}\]
and, with probability $1$,
\[ \ls\leq\inf_{\be\in\ldue\,\mid\,
  \nor{2}{\be}=1}\scal{2}{\Pn^*\Pn\be}{\be}.\]
Clearly we can have $\ls>0$ only if $\G$ is finite and the
features $\setG{\vp}$ are linearly independent both in $\Lduen$
and {$\Ldue$.}

On the other hand,\eqn{trace1} and (\ref{tracetx}) give the crude
upper spectral bounds
\begin{eqnarray*}
\sup_{\be\in\ldue\,\mid\,
  \nor{2}{\be}=1}\scal{2}{\PX^*\PX\be}{\be}& \leq & \ka, \\
\sup_{\be\in\ldue\,\mid\,
  \nor{2}{\be}=1}\scal{2}{\Pn^*\Pn\be}{\be}& \leq &\ka.
\end{eqnarray*}
{One} could improve these estimates by means of a tight bound
on the largest eigenvalue of $\PX^*\PX$.

We end this section by showing that, under the assumptions we
made, a structure of reproducing kernel Hilbert space emerges
naturally. Let us denote by $\Y^\X$ the space of functions from
$\X$ to $\Y$.
\begin{prop}\label{rkhs}
The linear operator $\Ph:\ldue\to\Y^\X$, $\Ph\be=f_\be$, is a
partial isometry from $\ldue$ onto the vector-valued reproducing
kernel Hilbert space $\hh$ on $\X$, with reproducing kernel
$K:\X\times \X\to {\mathcal L}(\Y)$ \beeq{rep_kernel} {K(x,t)y =
(\Ph_x\Ph_t^*) y =
  \sum_{\gG}\vg(x)\scaly{y}{\vg(t)}\qquad x,t\in\X,\ y\in\Y,}
the null space of $\Ph$ is
\beeq{kernel}{\Ker\Ph=\set{\be\in\ldue\mid \sum_{\gG}\vg(x)\bg=0\
  \  \forall x\in X}, }
and the family $\setG{\vp}$ is a {normalized tight} frame in $\hh$, namely
\[ \sum_{\gG}
|\scal{\hh}{f}{\vg}|^2=\nor{\hh}{f}^2\qquad \forall f\in\hh.\]
Conversely, let $\hh$ be a vector-valued reproducing kernel
Hilbert space with reproducing kernel $K$ such that $K(x,x):\Y\to
\Y$ is a trace-class operator for all $x\in\X$, with trace bounded
by $\ka$. If $\setG{\vp}$ is a {normalized tight} frame in $\hh$,
then\eqn{eq:rkhs} holds.
\end{prop}

\begin{proof}
Proposition~2.4 of \cite{cadeto06} (with ${\mathcal K}=\Y$,
$\widehat{\hh}=\ldue$, $\gamma(x)=\Ph_x^* $ and $A=\Ph$) gives
that $\Ph$ is a partial isometry from $\ldue$ onto the reproducing
kernel Hilbert space $\hh$, with reproducing kernel $K(x,t)$.
Eq.\eqn{kernel} is clear. Since $\Ph$ is a partial isometry with
range $\hh$ and $\Ph e_\g=\vg$ where $\setG{e}$ is a basis in
$\ldue$, then
$\setG{\vp}$ is {normalized tight} frame in $\hh$.\\
To show the converse result, given $x\in X$ and $y\in\Y$, we apply
the definition of a {normalized tight}
 frame {to the function $K_xy$ defined by $(K_xy)(t)=K(t,x)y$.
 $K_xy$  belongs to  $\hh$ by definition of a reproducing kernel Hilbert
 space and is such that the following reproducing property holds
 $\scal{\hh}{f}{K_xy}=\scaly{f(x)}{y}$ for any $f\in\hh$. Then }
\[ \scaly{K(x,x)y}{y}=\nor{\hh}{K_xy}^2=  \sum_{\gG}
|\scal{\hh}{K_xy}{\vg}|^2= \sum_{\gG}|\scaly{y}{\vg(x)}|^2,\]
where we used twice the reproducing {property.}
Now, if $(y_i)_{i\in I}$ is a basis in $\Y$ and $x\in\X$
\begin{multline*}
 \sum_{\gG}\nory{\vg(x)}^2 =\sum_{\gG}\sum_{i\in
  I}|\scaly{y_i}{\vg(x)}|^2=\sum_{i\in
  I} \scaly{K(x,x)y_i}{y_i}= \Tr{K(x,x)}\leq\ka.
  \end{multline*}
\end{proof}

%%%%%%%%%%%%%%%%%%%%%%%%%%%%%%%%%%%%%%%%%%%%%%%%%%%%%%%%%%%
\section{Minimization of the elastic-net  functional}\label{algos}
%%%%%%%%%%%%%%%%%%%%%%%%%%%%%%%%%%%%%%%%%%%%%%%%%%%%%%%%%%%
{In this section, we study the properties of the elastic net estimator $\est$ defined by\eqn{estimatore}. First of all, we characterize
the minimizer of the elastic-net  functional\eqn{elnetfunc}} as the unique fixed point of
a contractive map. Moreover, we characterize some sparsity
properties of the estimator and propose a natural iterative
soft-thresholding algorithm to compute it. {Our algorithmic
approach is totally different from the method proposed in
\cite{zhuhas05}, where $\est$ is computed by first reducing the
problem to the case of a pure $\ell_1$ penalty and then applying
the LARS algorithm \cite{efhajo04}.}

{In the following we make use the of the following vector notation.}
Given  a sample of $n$ i.i.d. observations
$(X_1,Y_1),\ldots,(X_n,Y_n)$,  and using the operators defined in
the previous section, we can rewrite {the elastic-net
  functional\eqn{elnetfunc} as}
\beeq{vecform}{ \In(\be)=\nor{n}{\Pn\be- Y}^2+ \rp \pen{\be},}
{where the $\pen{\cdot}$ is the elastic net penalty defined
by \eqn{enetpen}.}
%%%%%%%%%%%%%%%%%%%%%%%%%%%%%%%%%%%%%%%%%%%%%%%%%%%%%%%%%%%
\subsection{Fixed point equation}
%%%%%%%%%%%%%%%%%%%%%%%%%%%%%%%%%%%%%%%%%%%%%%%%%%%%%%%%%%%
The main difficulty in minimizing \eqn{vecform} is that the
functional is not differentiable because of the presence of the
$\ell_1$-term in the  penalty. Nonetheless the convexity of such
term  enables us to use tools from subdifferential calculus.
 Recall that, if $F: \ldue \to
\runo$ is a convex functional, the subgradient at a point
$\beta\in \ldue$ is the set of elements $\eta\in \ldue$
such that
$$
F(\beta+\beta') \geq F (\beta) + \scal{2}{\eta}{\beta'}\qquad
\forall \beta'\in \ldue .
$$
The subgradient at $\beta$ is denoted by $\partial F(\beta)$, see \cite{ekeland}. We  compute the subgradient of the
convex functional $\pen{\be}$, using the following definition of
$\sgn{t}$
\beeq{sub1}{\left\{\begin{array}{lcl}\sgn{t}=1 & \text{if } t>0 \\
\sgn{t}\in[-1,1] & \text{if } t=0 \\
\sgn{t}=-1 & \text{if } t<0.\end{array}\right. }

We first state the following lemma.
\begin{lemma}\label{penalty}
The functional $\pen{\cdot}$ is a convex, lower semi-continuous
(l.s.c.) functional from $\ldue$ into $[0,\infty]$. Given
$\be\in\ldue$, a vector $\eta\in\partial\pen{\be}$ if and only if
\[\eta_\g=\wg \sgn{\bg} + 2 \eps  \bg \ \ \forall \gG \qquad\text{and}\qquad \sum_{\gG}\eta_\g^2<+\infty .\]
\end{lemma}
\begin{proof}
Define the map $F:\G\times\runo\to [0,\infty]$
\[F(\g,t)= \wg |t| + \eps t^2.\]
Given $\gG$, $F(\g,\cdot)$ is a convex, continuous function and
its subgradient is
\[\partial F(\g,t)=\set{\tau\in\runo\mid \tau= \wg \sgn{t}+2 \eps  t},\]
where we used the fact that the subgradient of $|t|$ is given by
$\sgn{t}$. Since
\[\pen{\be}= \sum_{\gG} F(\g,\bg) =\sup_{\G'\text{ finite}} \sum_{\g\in \G'} F(\g,\bg)\]
and $\be\mapsto \bg$ is continuous, a standard result of convex
analysis \cite{ekeland} ensures that
$\pen{\cdot}$ is convex and lower semi-continuous.\\
The computation of the subgradient is standard. Given
$\be\in\ldue$ and $\eta\in\partial\pen{\be}\subset\ldue$, by the
definition of a subgradient,
\[\sum_{\gG} F(\g,\bg+\bg') \geq \sum_{\gG} F(\g,\bg) + \sum_{\gG}
\eta_\g \bg' \qquad \forall \be'\in\ldue.\] Given $\gG$, choose
$\be'=t e_\g$ with $t\in\runo$, it follows that $\eta_\g$ belongs
to the subgradient of $F(\g,\bg)$, that is, \beeq{1}{\eta_\g=\wg
\sgn{\bg} + 2 \eps  \bg.} Conversely, if\eqn{1} holds for all
$\gG$, by definition of a subgradient
\[F(\g,\bg+\bg')\geq  F(\g,\bg) + \eta_\g \bg'.\]
By summing over $\gG$ and taking into account the fact that
$\famG{\eta_\g\bg}\in\ell_1$, then
\[\pen{\be+\be'}\geq \pen{\be} +\scal{2}{\eta}{\be'}.\]
\end{proof}
To state our main result about the characterization of the
minimizer of \eqn{vecform}, we need to introduce the
soft-thresholding function ${\mathcal S}_\rp:\runo\to\runo$,
$\rp>0$ which is defined by
\beeq{thresholds}{\Ss{\rp}{t}=\left\{\begin{array}{lccl}
t-\frac{\rp}{2} &
\text{\rm if } & t>\frac{\rp}{2}  \\
0 & \text{\rm if } & |t|\leq \frac{\rp}{2} \\
t+\frac{\rp}{2} & \text{\rm if } & t <-\frac{\rp}{2}
\end{array}\right.,}
and the corresponding nonlinear thresholding operator ${\mathbf
S}_\rp:\ldue\to \ldue$ acting componentwise as
\beeq{threshold}{[\Sh{\rp}{\be}]_\g  =  \Ss{\rp\wg}{\bg}.}
We note that the soft-thresholding operator satisfies
\begin{eqnarray}
\Sh{a\rp}{a \be} & = & a \Sh{\rp}{\be}\qquad a>0, \be\in\ldue \label{dil1},\\
  \nor{2}{\Sh{\rp}{\be}- \Sh{\rp}{\be'}} & \leq &
  \nor{2}{\be-\be'}\qquad \be,\be'\in\ldue\label{lipS}.
\end{eqnarray}
These properties are immediate consequences of the fact that
\begin{eqnarray*}
  \Ss{a\rp}{a t} & = & a \Ss{\rp}{t}\qquad a>0, t\in\runo \\
  |\Ss{\rp}{t}- \Ss{\rp}{t'}| & \leq & |t-t'|\qquad t,t'\in\runo.
\end{eqnarray*}
Notice that\eqn{lipS} with $\be'=0$ ensures that
$\Sh{\rp}{\be}\in\ldue$ for all $\be\in\ldue$.

We are ready to prove the following theorem.
\begin{theo}\label{repre_n} 
Given $\eps\geq 0$ and $\rp>0$, a vector $\be\in\ldue$ is a
minimizer of {the elastic-net functional$\eqn{elnetfunc}$} if and only if it solves the
nonlinear equation 
\beeq{repres4}{\frac{1}{n}\sum_{i=1}^n
\scaly{Y_i -(\Pn\be)(X_i)}{\vg(X_i)}-\eps\rp\bg =\frac{\rp}{2} \wg
\sgn{\bg}\qquad \forall \gG,} 
or, equivalently, 
\beeq{repres3}{
\be=\Sh{\rp}{(1-\eps\rp)\be+\Pn^*(Y-\Pn\be)}.} If $\eps>0$ the
solution always exists and is unique. If $\eps=0$, $\ls>0$ and
$w_0=\inf_{\gG} \wg>0$, the solution still exists and is unique.
\end{theo}
\begin{proof}
If $\eps>0$ the functional $\In$ is  strictly convex, finite at
$0$,  and it is coercive by
\[\In(\be)\geq\pen{\be}\geq \rp \eps \nor{2}{\be}^2.\]
{Observing that $\nor{n}{\Pn\be-Y}^2$ is continuous and, by
  Lemma~\ref{penalty}, the elastic-net penalty is l.s.c., then $\In$ is
l.s.c. and, since} $\ldue$  is reflexive, there is a unique minimizer $\est$ in
$\ldue$.  If $\eps=0$, $\In$ is convex, but the fact that $\ls>0$
ensures that the minimizer is unique. Its existence follows from
the observation that
\[\In(\be)\geq\pen{\be}\geq \rp  \nor{1,w}{\be}\geq\rp w_0\nor{2}{\be},\]
where we used\eqn{daneo}.
In both cases the convexity of $\In$ implies that $\be$ is a
minimizer if and only if $0\in\partial\In(\be)$. Since
$\nor{n}{\Pn\be-Y}^2$ is continuous, Corollary III.2.1 of
\cite{ekeland} ensures that the subgradient is linear. Observing
that $\nor{n}{\Pn\be-Y}^2$ is differentiable with derivative
$2\Pn^*\Pn\be-2 \Pn^* Y$, we get
\[\partial\In(\be)= 2\Pn^*\Pn\be-2 \Pn^* Y + \rp \partial\pen{\be}.\]
Eq.\eqn{repres4} follows taking into account the explicit form of
$\partial\pen{\be}$, $\Pn^*\Pn \beta$ and $\Pn^*Y$, given by
Lemma~\ref{penalty} and 
Proposition~\ref{lemma0}, respectively. \\
We now prove\eqn{repres3}, which is equivalent to the set of
equations
 \beeq{repres5}{ \bg=\Ss{\rp\wg}{(1-\eps\rp) \bg+
\frac{1}{n}\sum_{i=1}^n  \scaly{Y_i
-(\Pn\be)(X_i)}{\vg(X_i)}}\qquad\forall\gG.} Setting $\bg'=
\scal{n}{Y-\Pn\be}{\vg(X)} -\eps\rp \bg$, we have
$\bg=\Ss{\rp\wg}{\bg+\bg'}$ if and only if
\[\bg=\left\{\begin{array}{lccl} \bg+\bg'- \frac{\rp\wg}{2} & \text{\rm if }& \bg+\bg'>
    \frac{\rp\wg}{2} \\
0 & \text{\rm if } & |\bg+\bg'|\leq \frac{\rp\wg}{2}\\
\bg+\bg'+ \frac{\rp\wg}{2} & \text{\rm if } & \bg+\bg'
<-\frac{\rp\wg}{2}  \end{array}\right.,
\]
that is,
\[\left\{\begin{array}{lcl} \bg'= \frac{\rp\wg}{2} & \text{if } \bg> 0 \\
|\bg'|\leq \frac{\rp\wg}{2}  & \text{if } \bg=0 \\
\bg' = - \frac{\rp\wg}{2} & \text{if } \bg <0
\end{array}\quad\hbox{ or else}\quad \bg'=\frac{\rp\wg}{2}\
\sgn{\bg}\right.
\]
which is equivalent to \eqn{repres4}.
\end{proof}
The following corollary gives some more information about the
characterization of the solution as the fixed point of a
contractive map. In particular, it provides an explicit expression
for the Lipschitz constant of this map and it shows how it depends
on the spectral properties of the empirical {mean of the} Gram matrix and on the
regularization parameter $\rp$.
\begin{corollary}\label{fixed} 
Let $\eps\geq 0$ and $\rp>0$. {Pick up any arbitrary
$\tau>0$. Then $\be$ is a minimizer of $\In$ in $\ldue$ if and
only if it is a fixed point of {the following Lipschitz map
$\T_n:\ldue\to \ldue$, namely}
 \beeq{n_contraction}{
\be=\T_n\be\qquad\text{where}\qquad \T_n\be=
\frac{1}{\tau+\eps\rp}\Sh{\rp}{(\tau I -\Pn^*\Pn)\be+ \Pn^* Y}.}
}
With the choice $\tau=\frac{\ls+\ka}{2}$, the Lipschitz constant
is bounded by
\[q= \frac{\ka-\ls}{\ka+\ls+ 2\eps\rp}\leq 1.\]
In particular, {with this choice of $\tau$ and} if $\eps>0$
or $\ls>0$, ${\T_n}$ is a contraction.
\end{corollary}

\begin{proof}
Clearly $\be$ is a minimizer of $\In$ if and only if it is a
minimizer of $\frac{1}{\tau+\eps\rp}\In$, which means that,
in\eqn{repres3}, we can replace $\rp$ with
$\frac{\rp}{\tau+\eps\rp}$, $\Pn$ by
$\frac{1}{\sqrt{\tau+\eps\rp}}\Pn$ and $Y$ by
$\frac{1}{\sqrt{\tau+\eps\rp}}Y$.  Hence $\be$ is a minimizer of
$\In$ if and only if it is a solution of
\[\be  =  \Sh{\frac{\rp}{\tau+\eps\rp}}{(1 -
\frac{\eps\rp}{\tau+\eps\rp})\be + \frac{1}{\tau+\eps\rp}
  \Pn^* (Y-\Pn\be)}.\]
Therefore, by\eqn{dil1} {with $a=\frac{1}{\tau+\eps\rp}$}, $\be$
is a minimizer of $\In$ if
  and only if $\be={\T_n}\be$.\\
We show that ${\T_n}$ is Lipschitz and calculate explicitly a bound on
the Lipschitz constant. By assumption we have $\ls I \leq
\Pn^*\Pn\leq \ka I$; then, by the Spectral Theorem,
\[ \norop{\ldue}{\ldue}{\tau I-\Pn^*\Pn}\leq\max\set{|\tau-\ls |, |\tau-\ka |},\]
{where $\norop{\ldue}{\ldue}{\cdot}$ denotes the operator norm of a bounded operator on $\ell_2$. Hence,} using \eqn{lipS}, we get
\begin{eqnarray*}
\nor{2}{{\T_n} \be-{\T_n} \be'} & \leq & \frac{1}{\tau+\eps\rp}
\nor{2}{(\tau I-\Pn^*\Pn) (\be- \be')} \\
& \leq & \max\set{|\frac{\tau-\ls }{\tau+\eps\rp}|,
  |\frac{\tau-\ka }{\tau+\eps\rp}|} \nor{2}{\be-\be'}\\
&=:&  q \nor{2}{\be-\be'}.
\end{eqnarray*}
The minimum of $q$ with respect to $\tau$ is obtained for
\[\frac{\tau-\ls }{\tau+\eps\rp}=\frac{\ka -\tau}{\tau+\eps\rp},\]
that is, $\tau= \frac{\ka +\ls }{2}$, and, with this choice, we
get
\[q= \frac{\ka -\ls }{\ka +\ls + 2\eps\rp}\ .\]
\end{proof}
By inspecting the proof, we notice that the choice of
$\tau=\frac{\ls+\ka}{2}$ provides the best possible Lipschitz
constant under the assumption that $\ls I\leq\Pn^*\Pn\leq\ka I$.
If $\eps>0$ or $\ls>0$, ${\T_n}$ is a contraction and
$\est$ can be computed by means of the Banach fixed point theorem.
If $\eps=0$ and $\ls=0$, ${\T_n}$ is only non-expansive, so that
proving the convergence of the successive approximation scheme is
not straightforward\footnote{Interestingly, it was proved in
\cite{daub} using different arguments that the same iterative
scheme can still be used for the case $\eps=0$ and $\ls=0$.}.

Let us now write down explicitly the iterative procedure suggested
by Corollary~\ref{fixed} to compute $\est$.
{Define} the iterative scheme by
\begin{eqnarray*}
\beta^0 & = & 0,\\
\be^\ell&=&
 \frac{1}{\tau+\eps\rp}\Sh{\rp}{(\tau I -\Pn^* \Pn)\be^{\ell-1}+ \Pn^* Y}\\
\end{eqnarray*}
with $\tau  = \frac{\ls+\ka}{2}$.
The following corollary shows that the $\be^\ell$
converges to $\est$ when $\ell$ goes to infinity.
\begin{corollary}\label{banach} 
{Assume that $\eps>0$ or $\ls>0$.}
For any $\ell\in\nat$ the following inequality holds
\beeq{banach_bound}{\nor{2}{\be^\ell-\est}\leq
 \frac{(\ka-\ls)^\ell}{(\ka+\ls+2\eps\rp)^{\ell}(\ls+\eps\rp)}  \nor{2}{\Pn^*Y}.}
{In particular,} $\lim_{\ell\to\infty}\nor{2}{\be^\ell-\est}=0$.
\end{corollary}
\begin{proof}
Since ${\T_n}$ is a contraction with Lipschitz constant $q=
\frac{\ka-\ls}{\ka+\ls+ 2\eps\rp}<1$, the Banach fixed point
theorem applies and the sequence
$\left(\be^\ell\right)_{\ell\in\nat}$ converges to the unique
fixed point of ${\T_n}$, which is $\est$ by Corollary~\ref{fixed}.
Moreover we can use the Lipschitz property of ${\T_n}$ to write
\begin{eqnarray*}
\nor{2}{\be^\ell -\est} &\le& \nor{2}{\be^\ell -\be^{\ell+1}} +
\nor{2}{\be^{\ell+1} -\est}\\
 &\le & q{ \nor{2}{\be^{\ell-1} -\be^{\ell}} + q
\nor{2}{\be^{\ell} -\est}}\\
&\le& q^\ell \nor{2}{\be^0 -\be^1} +q \nor{2}{\be^\ell -\est}\ ,
\end{eqnarray*}
so that we immediately get
\[
\nor{2}{\be^\ell -\est} \leq  \frac{q^{\ell}}{1-q}
\nor{2}{\be^1-\be^0}\leq
\frac{(\ka-\ls)^\ell}{(\ls+\ka+2\eps\rp)^{\ell}(\ls+\eps\rp)}
\nor{2}{\Pn^* Y}
\]
since $\be^0=0$, $\be^1=\frac{1}{\tau+\eps\rp}\Sh{\rp}{\Pn^*
  Y}$ and $1-q=\frac{2(\ls+\eps\rp)}{\ls+\ka+2\eps\rp}$.
\end{proof}

Let us remark that the bound \eqn{banach_bound} provides a natural
stopping rule for the number of iterations, namely to select
$\ell$ such that $\nor{2}{\be^\ell-\est}\leq{\eta}$, where
${\eta}$ is a bound on the distance between the estimator $\est$
and the true solution. For example, if $\nor{2}{\Pn^*Y}$ is
bounded by $M$ and if $\ls=0$, the stopping rule is
\[\ell_{\text{stop}}
\geq \frac{\log \frac{M}{\eps\rp{\eta}}}{\log (1+\frac{2
\eps\rp}{\ka})} \qquad\text{so that}\qquad
\nor{2}{\be^{\ell_{\text{stop}}}-\est}\leq{\eta}.\]

Finally we notice that all previous results also hold when
considering the distribution-dependent version of the method. The following
proposition summarizes the results in this latter case.
\begin{prop}\label{minimizer}
Let $\eps\geq 0$ and $\rp>0$. Pick up any arbitrary $\tau
>0$. Then a vector $\be\in\ldue$ is a minimizer of
\[\I(\be)= \ex{\nory{\PX \be -Y}^2} + \rp \pen{\be}.\]
if and only if it is a fixed point of the following Lipschitz map, namely
\beeq{contraction}{\be=\T\be\qquad\text{where}\qquad\T\be=
  \frac{1}{\tau+\eps\rp}\Sh{\rp}{(\tau
I -\PX^*\PX)\be+ \PX^* Y}.}
If $\eps>0$ or $\ls>0$, the minimizer
is unique.
\end{prop}
If it is unique, we denote it by $\mir$:

\beeq{def_mir}{\mir=\argmin_{\be\in\ldue}\left( \ex{\nory{\PX \be
        -Y}^2} + \rp \pen{\be}\right) .}
We add a comment. Under Assumption~\ref{statistical_model} and the
definition of $\beg$, the statistical model is $Y=\PX\beg+\noi$ where $\noi$ has zero mean,
so that $\mir$ is also the minimizer of
\begin{equation}
  \label{eq:simple_pop}
\inf_{\be\in\ldue}\left(\nor{P}{\PX \be -\PX\beg}^2 + \rp \pen{\be}\right).
\end{equation}

%%%%%%%%%%%%%%%%%%%%%%%%%%%%%%%%%%%%%%%%%%%%%%%%%%%%%%%%%%%%
\subsection{Sparsity properties}
%%%%%%%%%%%%%%%%%%%%%%%%%%%%%%%%%%%%%%%%%%%%%%%%%%%%%%%%%%%
The results of the previous section immediately yield a crude
estimate of the number and localization of the non-zero
coefficients of our estimator. Indeed, although the set of
features could be infinite, $\est$ has only a finite number of
coefficients different from zero provided that the sequence of
weights is bounded away from zero.
\begin{corollary}\label{monk}
Assume that the family of weights satisfies $\inf_{\gG}\wg>0$,
then for any $\be\in\ldue$, the support of $\Sh{\la}{\be}$ is
finite. In particular, $\est$, $\be^\ell$ and $\mir$ are all
finitely supported.
\end{corollary}

\begin{proof}
Let $w_0=\inf_{\gG}\wg>0$. Since $\sum_{\gG}|\be_\g|^2<+\infty$,
there is a finite subset $\G_0\subset\G$ such that $|\be_\g|\leq
\frac{\rp}{2} w_0\leq \frac{\rp}{2}\wg$ for all $\g\notin\G_0$.
This implies that
\[\Ss{\rp\wg}{\bg}=0\qquad\hbox{for}\qquad \g\not\in\G_0,\]
by the definition of soft-thresholding, so that the support of
$\Sh{\la}{\be}$ is contained in $\G_0$.
Equations\eqn{repres3},\eqn{contraction} and the definition of
$\be^\ell$ imply that $\est$, $\mir$ and $\be^\ell$  have finite
support.
\end{proof}

However, the supports of $\be^\ell$ and $\est$ are not known a
priori and to compute $\be^\ell$ one would need to store the
infinite matrix $\Pn^*\Pn$. The following corollary  suggests a
strategy to overcome this problem.
\begin{corollary}\label{trunc_dic}
Given $\eps\geq 0$ and $\rp>0$, let
$$
\G_\rp=\set{\gG\mid \nor{n}{\vg}\neq 0\text{ and }
\wg \leq \frac{2\nor{n}{Y}(\nor{n}{\vg}+  \sqrt{\eps\rp })}{\rp}}
$$
then
\beeq{suppest}{\supp{\est}\subset
\G_\rp.}
\end{corollary}

\begin{proof}
If $\nor{n}{\vg}=0$, clearly $\be_\g=0$ is a solution
of\eqn{repres4}. Let $M=\nor{n}{Y}$; the definition of $\est$ as
the minimizer of \eqn{vecform} yields the bound
$\In(\est)\leq\In(0) = M^2$, so that
\[\nor{n}{\Pn \est -Y}  \leq M\qquad \pen{\est}\leq
  \frac{M^2}{\rp}.\]
Hence, for all $\gG$, {the second inequality gives that }$\eps\rp
(\est)_\g^2 \leq M^2$, and we have
\[|\scal{n}{Y-\Pn\est}{\vg(X)}- \eps\rp  (\est)_\g|\leq M
(\nor{n}{\vg}+  \sqrt{\eps\rp })\] and, therefore, by
\eqn{repres4},
\[ |\sgn{(\est)_\g}|\leq  \frac{2M
(\nor{n}{\vg}+  \sqrt{\eps\rp })}{\rp \wg}.\] Since
$|\sgn{(\est)_\g}|=1$ when $(\est)_\g\neq 0$, this implies that
$(\est)_\g=0$ if $\frac{2M (\nor{n}{\vg}+  \sqrt{\eps\rp })}{\rp
\wg}<1$.
\end{proof}

Now, let $\G'$ be the set of indexes $\g$
such that the corresponding feature $\vg(X_i)\neq 0$ for some
$i=1,\ldots,n$. If the family of corresponding weights
$(\wg)_{\g\in\G'}$ goes to infinity\footnote{The sequence
  $(\wg)_{\g\in\G'}$ goes to infinity, if for all $ M>0$
there exists a finite set $\G_M$ such
that $|\wg|>M$, $\forall \g \notin \G_M$.}, then $\G_\rp$ is always finite.
Then,  since $\supp{\est}\subset \G_\rp$, one can replace $\G$ with
$\G_\rp$ in the definition of $\Pn$ so that $\Pn^*\Pn$ is a finite
matrix and $\Pn^*Y$ is a finite vector. In particular the
iterative procedure given by Corollary~\ref{fixed} can be
implemented by means of finite matrices.

Finally, by inspecting the proof above one sees that a similar
result holds true for the distribution-dependent minimizer $\mir$. Its
support is always finite, as already noticed, and moreover is
included in the following set
\[ \set{\gG\mid
  \nor{P}{\vg}\neq 0\text{ and }
  \wg \leq \frac{2\nor{\PP}{Y}(\nor{P}{\vg}+  \sqrt{\eps\rp })}{\rp}}.\]

%%%%%%%%%%%%%%%%%%%%%%%%%%%%%%%%%%%%%%%%%%%%%%%%%%%%%%%%%%%
\section{Probabilistic error estimates}\label{errorrs}
%%%%%%%%%%%%%%%%%%%%%%%%%%%%%%%%%%%%%%%%%%%%%%%%%%%%%%%%%%%
In this section we provide an error analysis for the elastic-net 
regularization scheme. Our primary goal is the {\em variable
selection problem}, so that we need to control 
the error $\nor{2}{\be_n^{\rp_n}-\be}$, where $\rp_n$ is a
suitable choice of the regularization parameter as a function of
the data, and $\be$ is an explanatory vector encoding the features
that are relevant to reconstruct the regression function $\f$,
that is, such that $\f=\PX\be$.  Although
Assumption\eqn{eq:sparse1} implies that the above equation has at
least a solution $\be^*$ with $\pen{\be^*}<\infty$, nonetheless,
the operator $\PX$ is injective only if $\famG{\vg(X)}$ is
$\ldue$-linearly independent in $\Ldue$. 
{As usually done for inverse problems, to restore uniqueness we choose, among all the vectors $\be$ such that $\f=\PX\be$, the vector $\beg$ which is the minimizer of the elastic-net  penalty. The vector $\beg$ can be regarded as the {\em best } representation of the regression function $\f$ according to the elastic-net penalty and we call it the {\em elastic-net representation}. Clearly this representation will depend on $\eps$.}

Next we focus on the following error decomposition (for any fixed
positive $\rp$), 
\beeq{deco}{\nor{2}{\est-\beg}\le
\nor{2}{\est-\mir}+\nor{2}{\mir-\beg}, } 
where $\mir$ is given by\eqn{def_mir}.
The first
error term in the {right-hand side }
of the above inequality is due to finite sampling and will be
referred to as the {\em sample error}, whereas the second error
term is deterministic and is called the {\em approximation error}.
In Section \ref{sec:sample_apprx} we analyze the sample error via
concentration inequalities and we consider the behavior of the
approximation error as a function of the regularization parameter
$\rp$. The analysis of these error terms leads us to discuss the
choice of $\rp$ and to derive statistical consistency results for elastic-net 
regularization. In Section \ref{sec:param} we discuss a priori and
a posteriori (adaptive) parameter choices.

%%%%%%%%%%%%%%%%%%%%%%%%%%%%%%%%%%%%%%%%%%%%%%%%%%%%%%%%%%%
\subsection{Identifiability condition and {elastic-net representation}}\label{sec:gen_sol}
%%%%%%%%%%%%%%%%%%%%%%%%%%%%%%%%%%%%%%%%%%%%%%%%%%%%%%%%%%%
The following proposition provides a way to define a unique
solution of the equation $\f=\PX\be$. Let
\[{\mathcal B} = \set{\be\in\ldue\mid \PX\be=\f(X)}= \be^* +\Ker{\PX} \]
where $\be^*\in\ldue$ is given by\eqn{eq:sparse1} in Assumption~\ref{statistical_model} and
\[\Ker{\PX}=\set{\be\in\ldue\mid \PX\be=0}={\set{\be\in\ldue\mid
    f_\be(X)=0\text{ with probability }1}}.\]
\begin{prop}\label{ora}
If $\eps>0$ or $\ls>0$, there is a unique $\beg\in\ldue$ such that
\beeq{generalized}{\pen{\beg}=\inf_{\be\in\mathcal B} \pen{\be}.}
\end{prop}

\begin{proof}
If $\ls>0$, $\mathcal B$ reduces to a single point, so that there
is nothing to prove. If $\eps>0$, $\mathcal B$ is a closed subset
of a reflexive space. Moreover, by Lemma~\ref{penalty}, the
penalty  $\pen{\cdot}$ is strictly convex, l.s.c. and, by\eqn{eq:sparse1} of Assumption~\ref{statistical_model}, there
  exists at least one $\be^*\in\mathcal B$ such that $\pen{\be^*}$ is finite.
Since $\pen{\be}\geq \eps \nor{2}{\be}^2$, $\pen{\cdot}$ is
coercive. A standard result of convex analysis implies that the
minimizer exists and is unique.
\end{proof}

%%%%%%%%%%%%%%%%%%%%%%%%%%%%%%%%%%%%%%%%%%%%%%%%%%%%%%%%%%%
\subsection{Consistency: sample and approximation errors}\label{sec:sample_apprx}
%%%%%%%%%%%%%%%%%%%%%%%%%%%%%%%%%%%%%%%%%%%%%%%%%%%%%%%%%%%
The main result of this section is a probabilistic error estimate
for $\nor{2}{\est-\mir}$, which will provide a choice $\la=\rp_n$
for the regularization parameter as well as a convergence result
for $\nor{2}{\be_n^{\rp_n}-\beg}$.

We first need to establish two lemmas. The first one shows that
the sample error can be studied in terms of the following
quantities \beeq{noises}{ \nor{{\text{HS}}}{\Pn^*\Pn -\PX^*\PX}
~~~\text{and}~~~\nor{2}{\Pn^*\noi} } measuring the perturbation
due to random sampling and noise (we recall that $\nor{\text{HS}}{\cdot}$
denotes the Hilbert-Schmidt norm of a Hilbert-Schmidt operator on
$\ldue$). The second lemma provides suitable probabilistic
estimates for these quantities.
\begin{lemma}
Let $\eps\geq 0$ and $\rp>0$. If $\eps>0$ or $\ls>0$, then
\begin{equation}
\label{eq:sample} \nor{2}{\est-\mir} \leq  \frac{1}{\ls+\eps\rp}
\left(\nor{2}{(\Pn^*\Pn -\PX^*\PX)(\mir-\beg)}+ \nor{2}{\Pn^*\noi}
\right ).
\end{equation}
\end{lemma}
\begin{proof}
Let $\tau=\frac{\ls+\ka}{2}$ and recall that $\est$ and $\mir$
satisfy\eqn{n_contraction} and\eqn{contraction}, respectively.
Taking into account\eqn{lipS} we get
\begin{equation}
  \label{eq:3}
\nor{2}{\est-\mir} \leq  \frac{1}{\tau+\eps\rp} \nor{2}{(\tau \est
-\Pn^*\Pn\est +\Pn^*Y) -(\tau\mir -
  \PX^*\PX\mir+\PX^* Y )}.
\end{equation}
By Assumption \ref{statistical_model} and the definition of
$\beg$, $Y=\f(X)+\noi$, and $\PX\beg$ and $\Pn\beg$ both coincide
with the function $\f$, {regarded as an element of $\Ldue$ and
  $\Lduen$ respectively}. Moreover by\eqn{eq:zeromean}
$\PX^*\noi=0$, so that
\[\Pn^*Y-\PX^*Y = (\Pn^*\Pn -\PX^*\PX)\beg +\Pn^*\noi. \]
Moreover
\[(\tau I  -\Pn^*\Pn)\est -(\tau I -
  \PX^*\PX)\mir= (\tau I  -\Pn^*\Pn)(\est-\mir)- (\Pn^*\Pn -\PX^*\PX)\mir.\]
From the assumption on $\Pn^*\Pn$ and the choice $\tau =
\frac{\ka+\ls}{2}$, we have $\norop{\ldue}{\ldue}{\tau
  I-\Pn^*\Pn}\leq\frac{\ka-\ls}{2}$, so that\eqn{eq:3} gives
\[{(\tau+\eps\rp)}\nor{2}{\est-\mir} \leq
\nor{2}{(\Pn^*\Pn -\PX^*\PX)(\mir-\beg)}+
  \nor{2}{\Pn^*\noi}+\frac{\ka-\ls}{2} \nor{2}{\est-\mir}.
\]
The bound \eqn{eq:sample} is established by observing that
$\tau+\eps\rp- (\ka-\ls)/2 = \ls+\eps\rp$.
\end{proof}
The probabilistic estimates for\eqn{noises} are straightforward
consequences of the law of large numbers for vector-valued random
variables. More precisely, we recall the following probabilistic
inequalities based on a result of \cite{pin94,pin99}; see also Th.
3.3.4 of \cite{yu} and \cite{pinsak85} for concentration
inequalities for Hilbert-space-valued random variables.
\begin{prop}\label{pine}
Let $(\xi_n)_{n\in\nat}$ be a sequence of i.i.d. zero-mean random
variables taking values in a real separable Hilbert space $\hh$
and satisfying \beeq{H}{ {\mathbb E}[\nor{\hh}{\xi_i }^m]\leq
\frac{1}{2} m! \Bb^2 \Ba^{m-2}\qquad \forall m\geq 2,} where $\Bb$
and $\Ba$ are two positive constants. Then, for all $n \in\nat$
and $\eta>0$ \beeq{92}{ \Prob{\nor{\hh}{\frac{1}{n} \sum_{i=1}^n
\xi_i}\geq \eta  }\leq  2 e^{-\frac{n\eta^2
    }{\Bb^2+\Ba \eta  +\Bb\sqrt{\Bb^2+2\Ba
      \eta  }} } =2
e^{-n\frac{\Bb^2}{\Ba^2}g(\frac{\Ba \eta }{\Bb^2})}}
where $g(t) =\frac{t^2}{1+t+\sqrt{1+2t},}$ or, for all $\delta>0$,
\beeq{confidence}{\Prob{\nor{\hh}{\frac{1}{n}
      \sum_{i=1}^n \xi_i}\leq
 \left( \frac{\Ba\delta}{n} +
  \frac{\Bb\sqrt{2\delta}}{\sqrt{n}}\right)}\geq 1-2e^{-\delta}.}
\end{prop}
\begin{proof}
Bound\eqn{92} is given in \cite{pin94} with a wrong factor, see
\cite{pin99}. To show\eqn{confidence},
observe that the inverse of the function
$\frac{t^2}{1+t+\sqrt{1+2t}}$ is the function $t+\sqrt{2t}$ so
that the equation
\[2 e^{-n\frac{\Bb^2}{\Ba^2}g(\frac{\Ba \eta }{\Bb^2})}=2e^{-\delta}\]
has the solution
\[\eta  =\frac{\Bb^2}{\Ba} \left(
  \frac{\Ba^2 \delta}{n\Bb^2} +
  \sqrt{2\frac{\Ba^2 \delta}{n\Bb^2}}\right).\]
\end{proof}

\begin{lemma}\label{noises_est}
With probability greater than $1-4 e^{-\delta}$,
{the} following inequalities hold, for any
$\la>0$ and $\eps>0$, \beeq{bound_noise}{ \nor{2}{\Pn^*\noi}\le
\left( \frac{L\sqrt{\ka}\delta}{n} +
  \frac{\sigma\sqrt{\ka}\sqrt{2\delta}}{\sqrt{n}}\right)\leq
{\underbrace{\frac{\sqrt{2\ka\delta}(\sigma+
      L)}{\sqrt{n}}}_{\text{if }\delta\leq n}}}
and
\beeq{brutto}{ \nor{HS}{\Pn^*\Pn -\PX^*\PX}\le \left(
\frac{\ka\delta}{n} +
  \frac{\ka\sqrt{2\delta}}{\sqrt{n}}\right)\leq
{\underbrace{\frac{3\ka\sqrt{\delta}}{\sqrt{n}}}_{\text{if }\delta\leq n}.}}
\end{lemma}
\begin{proof}
Consider the $\ell_2$ random variable $\Phi_X^*\noi$.
From\eqn{eq:zeromean},
$\ex{\Phi_X^*\noi}=\ex{\ex{\Phi_X^*\noi|X}}=0$ and, for any $m\geq
2$,
\[\ex{\nor{2}{\Phi_X^*\noi}^m}= \ex{(\sum_{\gG}\vert \scaly{\vg(X)}{\noi}\vert^2)^{\frac
  m2}}\leq \ka^{\frac{m}{2}} \ex{\nory{\noi}^m}\leq
\kappa^{\frac{m}{2}}\frac{m!}{2}\sigma^2 L^{m-2},\] due
to\eqn{eq:rkhs} and\eqn{eq:bennett}. Applying\eqn{confidence} with
$H=\sqrt{\ka}L$ and $M=\sqrt{\ka}\sigma$, and recalling the
definition \eqn{axs}, we get that
\[\Prob{\nor{2}{\Phi_n^*\noi}\leq \frac{\sqrt{\ka}L\delta}{n} +
  \frac{\sqrt{\ka}\sigma\sqrt{2\delta}}{\sqrt{n}}}\]
with probability greater than $1-2 e^{-\delta}$.\\
Consider the random variable $\Phi_X\Phi_X^*$ taking values in the
Hilbert space of Hilbert-Schmidt operators (where
$\nor{\mathrm{HS}}{\cdot}$ denotes the Hilbert-Schmidt norm). One
has that $\ex{\Phi_X\Phi_X^*}=\PX\PX^*$ and, by\eqn{trace}
\[\nor{\mathrm{HS}}{\Phi_X\Phi_X^*}\leq\Tr{\Phi_X\Phi_X^*}\leq \ka.\]
Hence
\begin{eqnarray*}
 \ex{\nor{\mathrm{HS}}{\Phi_X\Phi_X^*-\PX\PX^*}^m} & \leq &
\ex{\nor{\mathrm{HS}}{\Phi_X\Phi_X^*-\PX\PX^*}^2}(2\kappa)^{m-2} \\
 & \leq & \frac{m!}{2}\ka^2 \kappa^{m-2},
\end{eqnarray*}
by $m!\geq 2^{m-1}$. Applying\eqn{confidence} with $H=M=\kappa$
\[\Prob{ \nor{\mathrm{HS}}{\Phi_n\Phi_n^*-\PX\PX^*}}\leq \frac{\ka\delta}{n} +
  \frac{\ka\sqrt{2\delta}}{\sqrt{n}},\]
with probability greater than $1-2 e^{-\delta}$. The simplified
bounds {are clear provided that $\delta\leq n$.}
\end{proof}

{\begin{remark} In both\eqn{bound_noise} and\eqn{brutto}, the
condition $\delta\leq n$ allows to simplify the  bounds
enlightening the dependence on $n$ and the confidence level $1-4
e^{-\delta}$. In the following results we always assume that
$\delta\leq n$, but we stress the fact that this condition is only
needed to simplify the form of the bounds.  Moreover, observe
that, for a fixed confidence level, this requirement on $n$ is
very weak -- for example, {to achieve a  $99\%$ confidence level,
we only need to require that $n\geq 6$.}
  \end{remark}
}

 The next proposition gives a bound on the sample error. This
bound is uniform in the regularization parameter $\rp$ in the
sense that there exists an event independent of $\rp$ such that
its probability is greater than $1-4e^{-\delta}$
and\eqn{sample_err} holds true.
\begin{prop}\label{sample}
Assume that $\eps>0$ or $\ls>0$. 
{Let $\delta>0$ and $n\in\nat$ such that $\delta\leq n$, for any
  $\rp>0$ the bound
\beeq{sample_err}{\nor{2}{\est-\mir} \leq
  {\frac{\CC\sqrt{\delta}}{\sqrt{n}(\ls+\eps\rp)} \left(
      1+ \nor{2}{\mir-\beg}  \right)}}
holds with probability greater than $1-4e^{-\delta}$, where
$\CC=\max\set{\sqrt{2\ka}(\sigma+L),3\ka}$. }
\end{prop}
\begin{proof}
Plug bounds\eqn{brutto} and\eqn{bound_noise} in\eqn{eq:sample},
{taking into account that
\[\nor{2}{(\Pn^*\Pn-\PX^*\PX)(\mir-\beg)}\leq\nor{\text{HS}}{\Pn^*\Pn-\PX^*\PX}
\nor{2}{\mir-\beg}.\]
}
\end{proof}
\noindent{By inspecting the proof, one sees that the constant $\ls$ in\eqn{eq:sample} can be replaced by any constant $\ka_\la$ such that \[\ls\leq\ka_\la\leq\inf_{\be\in\ldue\mid\nor{2}{\be}=1}\nor{n}{\sum_{\g\in\G_\rp}\bg\vg}^2\qquad \text{with probability }1,\] 
where $\G_\rp$ is the set of {\em active features} given by Corollary~\ref{trunc_dic}. If $\ls=0$ and $\ka_\rp>0$, which means that $\G_\rp$ is finite and the active features are linearly independent, one can improve the bound\eqn{eq:5} below. Since we mainly focus on the case of linearly dependent dictionaries we will not discuss this point any further.\\ 
The following proposition shows that the approximation error $\nor{2}{\be^\rp-\beg}$ tends to zero when $\rp$ tends to zero. }
\begin{prop}\label{appr-prop}
{If $\eps>0$ }
then
\[\lim_{\rp\to 0}\nor{2}{\be^{\rp} -\beg}=0.\]
\end{prop}

\begin{proof}
It is enough to prove the result for an arbitrary sequence
$(\rp_j)_{j\in\nat}$ converging to 0. Putting $\be^j=\be^{\rp_j}$,
since $\nor{P}{\PX\be-Y}^2 =
\nor{P}{\PX\be-\f(X)}^2+\nor{P}{\f(X)-Y}^2$, by the definition of
$\be^j$ as the minimizer of \eqn{def_mir} and the fact that $\beg$
solves $\PX\be=\f$, we get
\[\nor{P}{\PX\be^j-\f(X)}^2 + \rp_j \pen{\be^j}  \leq
\nor{P}{\PX\beg-\f(X)}^2 + \rp_j \pen{\beg}= \rp_j \pen{\beg}.\]
Condition\eqn{eq:sparse1} of  Assumption~1 ensures that $\pen{\beg}$ is finite, so
that
\[\nor{P}{\PX\be^j-\f(X)}^2\leq \rp_j \pen{\beg}
\qquad\hbox{and}\qquad \pen{\be^j}\leq \pen{\beg}.\]
{Since $\eps>0$}, the last inequality implies that
$(\be^j)_{j\in\nat}$ is a bounded sequence
in $\ldue$. Hence, possibly passing to a subsequence,
$(\be^j)_{j\in\nat}$  converges weakly to some $\be_*$. We claim
that $\be_*=\beg$. Since $\be\mapsto \nor{P}{\PX \be-\f(X)}^2$ is
l.s.c.
\[\nor{P}{\PX \be_*-\f(X)}^2\leq\liminf_{j\to\infty}\nor{P}{\PX \be^j-\f(X)}^2\leq
\liminf_{j\to\infty} \rp_j \pen{\beg}=0,\] that is  $\be_*\in
{\mathcal B}$. Since $\pen{\cdot}$ is l.s.c.,
\[\pen{\be_*}  \leq  \liminf_{j\to\infty} \pen{\be^j}\leq
\pen{\beg}.\] By the definition of $\beg$, it follows that
$\be_*=\beg$ and, hence, \beeq{b1}{\lim_{j\to\infty}
\pen{\be^j}=\pen{\beg}.} To prove that $\be^j$ converges to $\beg$
in $\ldue$, it is enough to show that
{$\lim_{j\to\infty}\nor{2}{\be^j}= \nor{2}{\beg}$.
Since $\nor{2}{\cdot}$ is l.s.c.,
$\liminf_{j\to\infty}\nor{2}{\be^j}\geq\nor{2}{\beg}$. Hence we are
left to prove that }
$\limsup_{j\to\infty}\nor{2}{\be^j}\leq\nor{2}{\beg}$. Assume the
contrary. This implies that, possibly passing to a subsequence,
\[\lim_{j\to\infty} \nor{2}{\be^j}> \nor{2}{\beg} \]
and, using\eqn{b1},
\[\lim_{j\to\infty} \sum_{\gG}\wg |\bg^j| < \sum_{\gG}\wg |\beg|.\]
However, since  $\be\mapsto \sum_{\gG}\wg |\bg|$ is l.s.c.
\[\liminf_{j\to\infty} \sum_{\gG}\wg |\bg^j| \geq \sum_{\gG}\wg |\beg|.\]
\end{proof}
From \eqn{sample_err} and the triangular inequality, we easily
deduce that
\begin{eqnarray}
  \label{eq:5}
\nor{2}{\est-\beg} \leq {\frac{\CC\sqrt{\delta}  }{\sqrt{n}(\ls+\eps\rp)}
\left( 1+ \nor{2}{\mir-\beg}  \right)}
+ \nor{2}{\mir-\beg}
\end{eqnarray}
with probability greater that $1- 4e^{-\delta}$. Since the tails
are exponential, the above bound and the Borel-Cantelli lemma
imply the following theorem, which states that the estimator
$\est$ converges to the generalized solution $\beg$, for  a
suitable choice of the regularization parameter $\rp$.
\begin{theo}\label{teo_main} Assume that $\eps>0$ {and $\ls=0$.
{Let $\rp_n$ be a choice of $\rp$ as a function of $n$} such
that $\lim_{n\to\infty}\rp_n=0$ and
$\lim_{n\to\infty}n\rp_n^2-2\log n=+\infty$. Then
\[
\lim_{n\to\infty}\nor{2}{\be^{\rp_n}_{n}-\beg}=0\qquad\text{with
probability 1}.\] If $\ls>0$, the above convergence result holds
for any choice of $\rp_n$ such that $\lim_{n\to\infty}\rp_n=0$. }
\end{theo}
{\begin{proof} The only nontrivial statement concerns the
convergence with probability $1$. {We give the proof only for $\ls=0$, being the other one similar. Let $(\rp_n)_{n\geq 1}$ be a sequence such that  $\lim_{n\to\infty}\rp_n=0$ and
$\lim_{n\to\infty}n\rp_n^2-2\log n=+\infty$.}
Since
$\lim_{n\to\infty}\la_n=0$, Proposition~\ref{appr-prop} ensures
that $\lim_{n\to\infty}\nor{2}{\be^{\la_n}-\beg}=0$. Hence, it is
enough to show that
$\lim_{n\to\infty}\nor{2}{\be^{\rp_n}_{n}-\be^{\rp_n}}=0$ with
probability 1. {Let $D=\sup_{n\geq 1}\eps^{-1} \CC(1+\nor{2}{\be^{\rp_n}-\beg})$, which is finite since the approximation error goes to zero if $\la$ tends to zero. Given $\eta>0$, let $\delta=n\rp^2_n\frac{\eta^2}{D^2}\leq n$ for $n$ large enough, so that the bound\eqn{sample_err} holds providing that}
\[\Prob{\nor{2}{\be^{\rp_n}_{n}-\be^{\rp_n}}\geq\eta}\leq 4e^{-n\rp^2_n\frac{\eta^2}{D^2}}.\]
The condition that
$\lim_{n\to\infty}n\rp_n^2-2\log n=+\infty$ implies that the
series $\sum_{n=1}^\infty e^{-n\rp^2_n\frac{\eta^2}{D^2}}$
converges and the Borel-Cantelli lemma gives the thesis.
\end{proof}}

{\begin{remark}
The two conditions on $\rp_n$ in the above theorem are clearly
satisfied with the choice $\rp_n=(1/n)^r$ with $0<r<\frac{1}{2}$.
Moreover, by inspecting the proof, one can easily check that to
have the convergence of $\be_n^{\rp_n}$ to $\beg$ in probability,
it is enough to require that $\lim_{n\to\infty}\rp_n=0$ and
$\lim_{n\to\infty}n\rp_n^2=+\infty$.
  \end{remark}
}

Let $f_n=f_{\be_n^{\rp_n}}$. Since $\f=f_{\beg}$ and $\ex{\vert
f_n(X)-\f(X)\vert^2}=\nor{\PP}{\PX(\be_n^{\rp_n}-\beg)}^2$, the
above theorem  implies that
\[\lim_{n\to\infty}\ex{\vert f_n(X)-\f(X)\vert^2}=0\]
with probability 1, that is, the consistency of the elastic-net 
regularization scheme with respect to the square loss.

Let us remark that we are also able to prove such consistency
{without assuming\eqn{eq:sparse1} in Assumption~$\ref{statistical_model}$.
To this aim we need the following lemma, which is of interest by itself.}
{\begin{lemma}
Instead of Assumption~$\ref{statistical_model}$, assume that the
regression model is given by
\[Y=\f(X)+W,\]
where $\f:\X\to\Y$ is a bounded function and $W$
satisfies$\eqn{eq:zeromean}$ and$\eqn{eq:noise}$.
For fixed $\la$ and $\eps>0$, with probability greater than $1-2
e^{-\delta}$ we have
 \beeq{bella}{
\nor{2}{\Pn^*(f^\la-\f)-\PX^*(f^\la-\f)}\le \left(
    \frac{\sqrt{\ka}D_\rp\delta}{n} +
  \frac{\sqrt{2\ka\delta}\nor{\PP}{f^\rp-\f}}{\sqrt{n}}\right),}
where $f^\rp=f_{\be^\rp}$ and $D_\rp=\sup_{x\in\X}\nory{f^\rp(x)-\f(x)}$.
\end{lemma}
We notice that in\eqn{bella}, the function $f^\la-\f $ is regarded
both as an element of $\Lduen$ and as an element of $\Ldue$.}
\begin{proof}
Consider the {$\ldue$-valued} random variable
\[Z={\Px^*(f^\la(X)-\f(X))} \qquad Z_\g=\scaly{f^\rp(X)-\f(X)}{\vg(X)}.\]
A simple computation shows that
$\ex{Z}={\PX^*(f^\la-\f)}$ and
\[\nor{2}{Z} \leq \sqrt{\ka} \nory{f^\rp(X)-\f(X)}.\]
Hence, for any $m\geq 2$,
\[\begin{split}
\ex{\nor{2}{Z-\ex{Z}}^m} & \leq \ex{\nor{2}{Z-\ex{Z}}^2}
\left(2\sqrt{\ka}\sup_{x\in\X}\nory{f^\rp(x)-\f(x)}\right)^{m-2} \\
& \leq \ka
\ex{\nory{f^\rp(X)-\f(X)}^2}\left(2\sqrt{\ka}\sup_{x\in\X}\nory{f^\rp(x)-\f(x)}\right)^{m-2}\\
& \leq \frac{m!}{2}  (\sqrt{\ka} \nor{\PP}{f^\rp-\f})^2
(\sqrt{\ka}D_\rp)^{m-2}.
\end{split}\]
Applying\eqn{confidence} with $H=\sqrt{\ka}D_\rp$ and
$M=\sqrt{\ka}\nor{\PP}{f^\rp-\f}$, we obtain the bound\eqn{bella}.
\end{proof}
Observe that under Assumption\eqn{eq:sparse1} and by the
definition of $\beg$ one has that $D_\rp \leq
\sqrt{\ka}\nor{2}{\mir-\beg}$, so that\eqn{bella} becomes
\[\nor{2}{(\Pn^*\Pn-\PX^*\PX)(\mir-\beg)}\le \left({
    \frac{\ka\delta \nor{2}{\mir-\beg}}{n} +
  \frac{\sqrt{2\ka\delta}\nor{\PP}{\PX(\mir-\beg)}}{\sqrt{n}}}\right).\]
Since $\PX$ is a compact operator this bound is tighter than the
one deduced from\eqn{brutto}. However, the price we pay is that
the  bound does not hold uniformly in $\rp$. We are now able to
state the universal strong consistency of the elastic-net
regularization scheme.
\begin{theo}
Assume that $(X,Y)$ satisfy \eqn{eq:zeromean}
and\eqn{eq:noise} {and} that the regression function
$f^*$ is bounded. {If  the linear span of features $\setG{\vp}$ is
dense in $\Ldue$ and $\eps>0$}, then
\[\lim_{n\to\infty}\ex{\vert f_n(X)-\f(X)\vert^2}=0\qquad\text{with
probability 1},\]
provided that $\lim_{n\to\infty} \rp_n=0$ and
{$\lim_{n\to\infty}n\rp_n^2-2\log n=+\infty$}.
\end{theo}
\begin{proof}
As above we bound separately the approximation error and the
sample error. As for the first term, let $f^\rp=f_{\mir}$. We
claim that $\ex{\vert f^\rp(X)-\f(X)\vert^2}$ goes to zero when
$\rp$ goes to zero. Given $\eta>0$, the fact that the {linear span
  of the features $\setG{\vp}$ is dense} in $\Ldue$ implies that there
is $\be^\eta\in\ldue$ such that  $\pen{\be^\eta}<\infty$ and
\[\ex{\vert f_{\be^\eta}(X)-Y\vert^2}\leq \ex{\vert\f(X)-Y\vert^2}+\eta .\]
Let $\rp_{\eta}=\frac{\eta}{1+\pen{\be^\eta}}$, then, for any
$\rp\leq\rp_{\eta}$,
\[
\begin{split}
 \ex{\vert f^\rp(X)-\f(X)\vert^2} & \leq\left( \ex{\vert f^\rp(X)-Y\vert^2}
 -\ex{\vert\f(X)-Y\vert^2}\right) + \rp
 \pen{\mir} \\
& \leq \left( \ex{\vert
f_{\be^\eta}(X)-Y\vert^2}-\ex{\vert\f(X)-Y\vert^2}\right) + \rp
 \pen{\be^\eta} \\
&\leq  \eta  + \eta\, .
\end{split}
\]
As for the sample error, we let $f_n^\rp=f_{\est}$ (so that
$f_n=f_n^{\rp_n}$) and observe that
\[\ex{\vert f^\rp(X)-f_n^\rp(X)\vert^2}=\nor{\PP}{\PX(\est-\mir)}^2\leq \ka\nor{2}{\est-\mir}^2.\]
We bound $\nor{2}{\est-\mir}$ {by\eqn{bella}}
observing that
\[
\begin{split}
D_\rp & =\sup_{x\in\X}\nory{f^\rp(x)-\f(x)} \leq\sup_{x\in\X}\nory{f_{\be^\rp}(x)}+\sup_{x\in\X}\nory{\f(x)} \\
 & \leq  \sqrt{\ka}\nor{2}{\be^\rp}      +  \sup_{x\in\X}\nory{\f(x)}
 \leq  D\frac{1}{\sqrt{\rp}}
\end{split}
\]
where $D$ is a suitable constant and where we used the crude
estimate
\[\rp\eps\nor{2}{\mir}^2\leq \mathcal{E}^\rp(\mir)\leq\mathcal{E}^\rp(0)=\ex{\nory{Y}^2}.\]
Hence\eqn{bella} yields
{
\begin{equation}
  \label{44}
 \nor{2}{\Pn^*(f^\la-\f)-\PX^*(f^\la-\f)}\le \left(
    \frac{\sqrt{\ka}\delta D}{\sqrt{\rp}n} +
  \frac{\sqrt{2\ka\delta}\nor{\PP}{f^\rp(X)-\f(X)}}{\sqrt{n}}\right).
\end{equation}
Observe that the proof of \eqn{eq:sample} does not  depend on the
  existence of $\beg$ provided that we replace both $\PX\est\in\Ldue$ and
  $\Pn\est\in\Lduen$ with $\f$, and we take into account that both $\PX\mir\in\Ldue$
  and $\Pn\mir\in\Lduen$ are equal to $f^\la$. Hence,
  plugging\eqn{44}} and\eqn{bound_noise} in\eqn{eq:sample} we have that with
probability greater than $1- 4 e^{-\delta}$
\[\nor{2}{\est-\mir} \leq \frac{D{\sqrt{\delta}}}{\ls+\eps\rp} \left( \frac{1}{\sqrt{n}}+
   \frac{1}{\sqrt{\rp}n} +
  \frac{\nor{\PP}{f^\rp(X)-\f(X)}}{\sqrt{n}}  \right)
\]
where $D$ is a suitable constant and {$\delta\leq n$}. The thesis now follows by
combining the bounds on the sample and approximation errors {and
  repeating the proof of Theorem~\ref{teo_main}.}
\end{proof}
To have an explicit convergence rate, one needs a explicit bound
on the approximation  error $\nor{2}{\mir-\beg}$, for example of
the form $\nor{2}{\mir-\beg}=O(\la^r)$. This is out of the scope
of the paper. We report only the following simple result.
\begin{prop}
Assume that {the features $\vg$ are in finite number and
  linearly independent. Let $N^*=|\supp{\beg}|$ and $w^*=\sup_{\g\in\supp{\beg}}\set{\wg}$, then
\[\nor{2}{\mir-\beg} \leq D N^* \rp.\] }
With the choice $\rp_n=\frac{1}{\sqrt{n}}$, {for any $\delta>0$
  and $n\in\nat$ with $\delta\leq n$}
\begin{equation}
\label{eq:1}
\nor{2}{\be^{\rp_n}_{n}-\beg}\leq
\frac{{\CC\sqrt{\delta}}}{\sqrt{n}\ls}
\left( {1}+\frac{DN^*}{\sqrt{n}}\right) +
\frac{DN^*}{\sqrt{n}},
\end{equation}
with probability greater than $1-4e^{-\delta}$, where
$D=\frac{{w^*}}{2\ls} + \eps \nor{\infty}{\beg}$ and
{$\CC=\max\set{\sqrt{2\ka}(\sigma+L),3\ka}$}.
\end{prop}
\begin{proof}
{Observe that the assumption on the set of features is equivalent to assume
  that $\ls>0$.} First, we bound the approximation error $\nor{2}{\mir-\beg}$. As
usual, with the choice $\tau=\frac{\ls+\ka}{2}$,
Eq.\eqn{contraction} gives
\begin{eqnarray*}
\mir-\beg =\frac{1}{\tau+\eps\rp}\left[ \Sh{\rp}{(\tau I
    -\PX^*\PX)\mir+ \PX^*\PX\beg}-\Sh{\rp}{\tau\beg}+\Sh{\rp}{\tau\beg}
  - \tau\beg\right]- \frac{\eps\rp}{\tau+\eps\rp} \beg.
\end{eqnarray*}
Property\eqn{lipS} implies that
\begin{eqnarray*}
\nor{2}{\mir-\beg} & \leq & \frac{1}{\tau+\eps\rp}\left(
\nor{2}{(\tau I
    -\PX^*\PX)(\mir-\beg)} + \nor{2}{\Sh{\rp}{\tau\beg} - \tau\beg}\right) \\
& & +  \frac{\eps\rp}{\tau+\eps\rp} \nor{2}{\beg}.
\end{eqnarray*}
Since $\Vert \tau I-\PX^*\PX\Vert \leq \frac{\ka-\ls}{2} $,
$\nor{2}{\beg} \leq N^*\nor{\infty}{\beg}$ and
\[\nor{2}{\Sh{\rp}{\tau\beg} - \tau\beg}\leq {w^*}N^*\frac{\rp}{2} , \]
one has
\begin{eqnarray*}
 \nor{2}{\mir-\beg} & \leq & \frac{\ka+\ls+2\eps\rp}{2(\ls+\eps\rp)}
 \left(\frac{2}{\ka+\ls+2\eps\rp} {w^*}N^*\frac{\rp}{2} +
  \frac{2\eps\rp}{\ls+\ka+2\eps\rp} \nor{2}{\beg}\right)  \\
& \leq & (\frac{{w^*}}{2\ls} + \eps \nor{\infty}{\beg})N^*\rp = D
N^*\rp.
\end{eqnarray*}
The bound\eqn{eq:1} is then an straightforward consequence
of\eqn{eq:5}.
\end{proof}
Let us observe this bound is weaker than the results obtained in
\cite{kolt06} since the constant $\ls$ is a {\em global} property
of the dictionary, whereas the constants in \cite{kolt06} are {\em
local}.

%%%%%%%%%%%%%%%%%%%%%%%%%%%%%%%%%%%%%%%%%%%%%%%%%%%%%%%%%%%
\subsection{Adaptive choice}\label{sec:param}
%%%%%%%%%%%%%%%%%%%%%%%%%%%%%%%%%%%%%%%%%%%%%%%%%%%%%%%%%%%
In this section, we suggest an adaptive choice of the
regularization parameter $\rp$. The main advantage of this
selection rule is that it does not require {any}
{knowledge of the behavior of the approximation error}. To
this aim, it is useful to replace the approximation error with the
following upper bound
\begin{equation}
  \label{eq:app_bound}
\app{\rp}=\sup_{0<\rp'\leq\rp}\nor{2}{\be^{\rp'}-\beg}.
\end{equation}
The following simple result holds.
\begin{lemma}\label{decr}
Given $\eps>0$, $\mathcal A$ is an increasing continuous function
and
\begin{eqnarray*}
 \nor{2}{\mir-\beg}\leq\app{\la}\leq A<\infty \\
\lim_{\la\to 0+}\app{\la}=0.
\end{eqnarray*}
\end{lemma}
\begin{proof}
First of all, we show that $\rp\mapsto\mir$ is a continuous
function. Fixed $\rp>0$, for any $h$ such that $\rp+h>0$, Eq.
\eqn{contraction}  with $\tau=\frac{\ls+\ka}{2}$ and
Corollary~\ref{fixed} give
\[
\begin{split}
\nor{2}{\be^{\rp+h}-\mir} & \leq\nor{2}{{\mathcal
    T}_{\rp+h}(\be^{\rp+h}) - {\mathcal T}_{\rp+h}(\mir)}  + \nor{2}{{\mathcal T}_{\rp+h}(\mir) - {\mathcal T}_{\rp}(\mir)}\\
& \leq \frac{\ka-\ls}{\ka+\ls+ 2\eps(\rp+h)}
\nor{2}{\be^{\rp+h}-\mir}  +
\\
& \quad + \nor{2}{\frac{1}{\tau+\eps(\rp+h)}\Sh{\rp+h}{\be'}-
  \frac{1}{\tau+\eps\rp}
\Sh{\rp}{\be'}}
\end{split}
\]
where $\be'=(\tau I -\PX^*\PX)\mir+ \PX^* Y$ does not depend on
$h$ and {we wrote ${\mathcal T}_{\rp}$  to make explicit the
dependence of the map ${\mathcal T}$ on the regularization
parameter}. Hence
\[
\begin{split}
\nor{2}{\be^{\rp+h}-\mir} &\leq \frac{\tau+
  \eps(\rp+h)}{\ls+\eps(\rp+h)}\left(\left|\frac{1}{\tau+\eps(\rp+h)}
  -\frac{1}{\tau+\eps\rp}\right|\right.\nor{2}{\be'} + \\
& \quad + \left. \frac{1}{\tau+\eps\rp}\nor{2}{\Sh{\rp+h}{\be'}-
  \Sh{\rp}{\be'}}\right).
\end{split}
\]
The claim follows by observing that {(assuming for
simplicity that $h > 0$)} {\[
\begin{split}
\nor{2}{\Sh{\rp+h}{\be'}- \Sh{\rp}{\be'}}^2
& = \sum_{\wg\rp\leq|\bg'|<\wg(\rp+h)} |\bg'-\sgn{\bg'}\wg\rp|^2 + 
 \sum_{|\bg'|\geq\wg(\rp+h)} \wg^2 h^2 \\
& \leq h^2 \sum_{|\bg'|\geq\wg\rp} \wg^2 \leq h^2
\sum_{|\bg'|\geq\wg\rp} (\bg'/\rp)^2 \leq h^2
\nor{2}{\be'}^2/\rp^2\ ,
\end{split}
\]
{which goes to zero if $h$ tends to zero.\\}}
Now, by the definition of $\mir$ and $\beg$
\[\eps\rp \nor{2}{\mir}^2 \leq \ex{\nory{\PX \mir -\f(X)}^2} + \rp
\pen{\mir}\leq \ex{\nory{\PX \beg -\f(X)}^2} + \rp
\pen{\beg}=\rp\pen{\beg},\] so that
\[\nor{2}{\mir-\beg}\leq \nor{2}{\beg}+\frac{1}{\sqrt{\eps}}\pen{\beg}=:A.\]
Hence $\app{\rp}\leq A$ for all $\rp$. Clearly $\app{\rp}$ is an
increasing function of $\rp$; the fact that $\nor{2}{\mir-\beg}$
is continuous and goes to zero with $\rp$ ensures that the same
holds true for $\app{\rp}$.
\end{proof}
Notice that we replaced the approximation error with $\app{\rp}$
just for a technical reason, namely to deal with an increasing
function of $\rp$. If we have a monotonic decay rate at our
disposal, such as $\nor{2}{\mir-\beg}\asymp\rp^a$ for some $a>0$
and for $\lambda \to 0$, then clearly $\app{\rp}\asymp\rp^a$.

Now, we fix $\eps>0$ and $\delta\geq 2$ and we assume that $\ls=0$. Then we
simplify the bound \eqn{eq:5} observing that
\begin{equation}
  \label{semplice}
\nor{2}{\est-\beg} \leq C\left(\frac{1}{\sqrt{n}\eps\rp} +
  \app{\rp}\right)
\end{equation}
where $C={\CC\sqrt{\delta}(1+A)}$; the bound holds with probability greater than  $1-4e^{-\delta}$
uniformly for all $\rp>0$. \\
When $\rp$ increases, the first term in\eqn{semplice} decreases
 whereas the second increases; hence {to have a tight bound
a {\em natural} choice of the parameter 
consists in balancing the two terms in the above bound}, namely in
taking
\[\rpopt=\sup\set{\rp\in ]0,\infty[\ \mid \app{\rp}=\frac{1}{\sqrt{n}\eps\rp}}.\]
Since $\app{\rp}$ is continuous, $\frac{1}{\sqrt{n}\eps\rpopt}
=\app{\rpopt}$ and the resulting bound is
\begin{equation}
  \label{best_rate}
\nor{2}{\est-\beg} \leq \frac{2C}{\sqrt{n}\eps\rpopt}.
\end{equation}
This method for choosing the regularization parameter clearly
requires the knowledge of the approximation error. To overcome
this drawback, we discuss a data-driven choice for $\rp$ that
allows to achieve the rate \eqn{best_rate} {\em without} requiring
any prior information on $\app{\rp}$. For this reason, such choice
is said to be {\em adaptive}. The procedure we present is also
referred to as an {\em a posteriori} choice since it depends on
the given sample and not only on its cardinality $n$. In other
words, the method is purely data-driven.

Let us consider a discrete set of values for $\rp$ defined by the
geometric sequence
$$
\rp_i=\rp_{0} 2^i \qquad i\in\nat \qquad \rp_0>0.
$$
Notice that we may replace the sequence $\rp_{0} 2^i$ be any other
geometric sequence $\rp_i=\rp_{0} q^i$ with $q>1$; this would only
lead to a more complicated constant in\eqn{finale}.
Define the parameter  $\rp^+_n$ as follows
 {\beeq{balance}{
\rp^+_n=\max\{ \rp_i| \nor{2}{\be_n^{\rp_{j}}-
\be_n^{\rp_{j-1}}}\le \frac{4C}{\sqrt{n}\eps\rp_{j-1}}\text{ for all
}j=0,\ldots,i \}}} (with the
convention that $\rp_{-1}=\rp_0$).  This strategy for choosing
$\lambda$ is inspired by a procedure originally proposed in
\cite{lep90} for Gaussian white noise regression and which has
been widely discussed in the context of deterministic as well as
stochastic inverse problems (see \cite{bauper05,persch03}). In the
context of nonparametric regression from random design, this
strategy has been considered in \cite{depero07} and the following
proposition is a simple corollary of a result contained in
\cite{depero07}.
\begin{prop}
Provided that $\rp_0<\rpopt$, the following bound holds with
probability greater than $1-4e^{-\delta}$
\begin{equation}
  \label{finale}
\nor{2}{\be_n^{\rp^+_n}-\beg}\le \frac{20C}{\sqrt{n}\eps\rpopt}.
\end{equation}
\end{prop}
\begin{proof}
The proposition results from Theorem~2 in \cite{depero07}. For
completeness, we report here a proof adapted to our setting. Let
$\Omega$ be the event such that\eqn{semplice} holds for any
$\rp>0$; we have that ${\mathbb
  P}[\Omega]\geq 1 -4 e^{-\delta}$ and we fix a sample point in
$\Omega$.\\
The definition of $\rpopt$ and the assumption $\rp_0<\rpopt$
ensure that $\app{\rp_0}\leq \frac{1}{\sqrt{n}\eps\rp_0}$. Hence
the set $\set{\rp_i\mid \app{\rp_i}\leq
  \frac{1}{\sqrt{n}\eps\rp_i}}$ is not empty and we can define
\[\rp^*_n=\max\set{\rp_i\mid \app{\rp_i}\leq
  \frac{1}{\sqrt{n}\eps\rp_i}}.\]
The fact that $(\rp_i)_{i\in\nat}$ is a geometric sequence implies
that
\begin{equation}
  \label{eq:7}
\rp^*_n \leq    \rpopt< 2\rp^*_n,
\end{equation}
while \eqn{semplice} with the definition of $\rp_n^*$ ensures that
\begin{equation}
\label{eq:8}
\nor{2}{\be_n^{\rp^*_n}-\beg}  \leq C
\left(\frac{1}{\sqrt{n}\eps\rp^*_n} +
  \app{\rp^*_n}\right) \leq   \frac{2C}{\sqrt{n}\eps\rp^*_n}\, .
\end{equation}
We show that $\rp^*_n\leq \rp^+_n$. Indeed, for any
$\rp_j<\rp^*_n$, using\eqn{semplice}  twice, we get
\[
\begin{aligned}
\nor{2}{\be_n^{\rp^*_n}-\be_n^{\rp_j}} &\leq
\nor{2}{\be_n^{\rp_j}-\beg} +  \nor{2}{\be_n^{\rp^*_n}-\beg}\\
&\leq C \left(\frac{1}{\sqrt{n}\eps\rp_j} +
  \app{\rp_j}+ \frac{1}{\sqrt{n}\eps\rp^*_n} +
  \app{\rp^*_n}\right)\leq \frac{4C}{\sqrt{n}\eps\rp_j}\ ,
\end{aligned}
\]
where the last inequality holds since $\rp_j<\rp^*_n\leq\rpopt$
and $\app{\la}\leq \frac{1}{\sqrt{n}\eps\rp}$ for all
$\rp<\rpopt$. Now $2^m \la_0 \leq\rp^*_n\leq \rp_n^+=2^{m+k}$ for
some $m,k\in\nat$, so that
\[
\begin{split}
\nor{2}{\be_n^{\rp_n^+}-\be_n^{\rp^*_n}} & \leq
\sum_{\ell=0}^{k-1}
\nor{2}{\be_n^{m+1+\ell}- \be_n^{m+\ell} }
 \leq \sum_{\ell=0}^{k-1}\frac{4C}{\sqrt{n}\eps\rp_{m+\ell}} \\
& \leq \frac{4C}{\sqrt{n}\eps\rp^*_n}\sum_{\ell=0}^{\infty}
\frac{1}{2^\ell}  = \frac{4C}{\sqrt{n}\eps\rp^*_n}\ 2\ .
\end{split}
\]
Finally, recalling\eqn{eq:7} and\eqn{eq:8}, we get the bound
\eqn{finale}:
\[
\nor{2}{\be_n^{\rp_n^+}-\beg}  \leq
\nor{2}{\be_n^{\rp_n^+}-\be_n^{\rp^*_n}} +
\nor{2}{\be_n^{\rp^*_n}-\beg}  \leq \frac{8C}{\sqrt{n}\eps\rp^*_n}
+
 \frac{2C}{\sqrt{n}\eps\rp^*_n}
\leq  20C\frac{1}{\sqrt{n}\eps\rpopt}.
\]
\end{proof}
Notice that the {\em a priori} condition $\la_0<\rpopt$ is
satisfied, for example, if $\la_0 < \frac{1}{A\eps\sqrt{n}}$.

To illustrate the implications of the last Proposition, let us
suppose that \beeq{appr_r}{\nor{2}{\mir-\beg}\asymp\rp^a} for some
unknown $a\in ]0,1]$. One has then that $\rpopt\asymp
n^{-\frac{1}{2(a+1)}} $ and $ \nor{2}{\be_n^{\rp^+_n}-\beg}\asymp
n^{-\frac{a}{2(a+1)}}$.

We end noting that, if we specialize our analysis to least squares
regularized with a pure $\ell_2$-penalty (i.e. setting $\wg=0$,
$\forall \gG$), then our results lead to the error estimate in the
norm of the reproducing kernel space $\hh$ obtained in
\cite{smazho05,bapero05}. Indeed, in such a case, $\beg$ is the
generalized solution $\be^\dag$ of the equation $\PX\be=\f$
and the approximation error satisfies\eqn{appr_r} under the a
priori assumption that the regression vector $\be^\dag$ is in the
range of $(\PX^* \PX)^a$ for some $0<a\leq 1$ (the fractional
power makes sense since $\PX^*\PX$ is a positive operator). Under
this assumption, it follows that
$\nor{2}{\be_n^{\rp^+_n}-\beg}\asymp n^{-\frac{a}{2(a+1)}}$. To
compare this bound with the results in the literature, recall that
both $f_n=f_{\be_n^{\rp^+_n}}$ and $\f=f_{\be^\dag}$ belongs to
the reproducing kernel Hilbert space $\hh$ defined in
Proposition~\ref{rkhs}. In particular, one can check that
$\be^\dag\in\operatorname{ran}{(\PX^* \PX)^a}$ if and only if
$\f\in\operatorname{ran}{L_K^{\frac{2a+1}{2}}}$, where
$L_K:\Ldue\to\Ldue$ is the integral operator whose kernel is the
reproducing kernel $K$ \cite{capdev05}. Under this condition, the following bound
holds
\[\nor{\hh}{f_n-\f}\leq \nor{2}{\be_n^{\rp^+_n}-\beg}\asymp
n^{-\frac{a}{2(a+1)}},\] which gives the same rate as in {Theorem~2 of
\cite{smazho05} and Corollary~17 of \cite{bapero05}.}

\section*{Acknowledgments}
We thank Alessandro Verri for helpful suggestions and discussions.
Christine De Mol acknowledges support by the ``Action de Recherche
Concert\'ee'' Nb 02/07-281, the VUB-GOA 62 grant and the National
Bank of Belgium BNB; she is also grateful to the DISI,
Universit\`a di Genova for hospitality during a semester in which
the present work was initiated. Ernesto De Vito and Lorenzo Rosasco have been
partially supported by the FIRB project
  RBIN04PARL and by the the EU Integrated Project Health-e-Child
  IST-2004-027749.

%\bibliographystyle{abbrv}
%\bibliography{biblio}

\end{document}